\documentclass[11pt]{article}

\usepackage[a4paper]{geometry}
\usepackage{subcaption}

\usepackage{amssymb}
\usepackage{amsmath,amsfonts,amscd,amsthm,xspace}
\usepackage{authblk}
\usepackage{epsfig}
\usepackage{wrapfig}
\usepackage{enumitem}
\usepackage{dblfloatfix}
\usepackage{array}
\usepackage{mathtools}
\usepackage{subcaption}
\usepackage{fixltx2e}
\usepackage{color}
\usepackage[utf8]{inputenc}
\usepackage[ruled]{algorithm2e}
\usepackage[T1]{fontenc}
\usepackage{url}
\usepackage{booktabs}
\usepackage{amsfonts}
\usepackage{nicefrac}
\usepackage{microtype}   
\usepackage{graphicx}
\graphicspath{{Figs/}}

\usepackage{newtxtext}

\usepackage{multirow}
\usepackage{xcolor}         
\usepackage{algpseudocode}

\newtheorem{theorem}{Theorem}
\newtheorem{lemma}{Lemma}
\newtheorem*{nonumbertheorem}{Theorem}
\newtheorem{example}{\textit{Example}}

\newtheorem{definition}{Definition}

\newtheorem{assumption}{Assumption}

\newcommand{\nn}{n}
\newcommand{\nm}{m}

\newcommand{\tabincell}[2]{\begin{tabular}{@{}#1@{}}#2\end{tabular}}

\providecommand{\keywords}[1]
{
  \textbf{\textit{Keywords---}} #1
}

\title{Equipping Black-Box Policies with Model-Based Advice for Stable Nonlinear Control}

\author[1]{
  Tongxin~Li}
\author[2]{Ruixiao Yang} \author[3]{Guannan Qu}  \author[1]{Yiheng Lin}  \author[1]{Steven Low}
\author[1]{Adam Wierman}
\affil[1]{California Institute of Technology}
\affil[2]{Tsinghua University}
\affil[3]{Carnegie Mellon University}

\begin{document}

\maketitle

\begin{abstract}
Machine-learned black-box policies are ubiquitous for nonlinear control problems. 
    Meanwhile, crude model information is often available for these problems from, e.g., linear approximations of nonlinear dynamics.
    We study the problem of equipping a black-box control policy with model-based advice for nonlinear control on a single trajectory.  We first show a general negative result that a naive convex combination of a black-box policy and a linear model-based policy can lead to instability, even if the two policies are both stabilizing.
    We then propose an \textit{adaptive $\lambda$-confident policy}, with a coefficient $\lambda$ indicating the confidence in a black-box policy, and prove its stability. 
   With bounded nonlinearity, in addition, we show that the adaptive $\lambda$-confident policy achieves a bounded competitive ratio when a black-box policy is near-optimal. Finally, we propose an online learning approach to implement the adaptive $\lambda$-confident policy and verify its efficacy in case studies about the CartPole problem and a real-world electric vehicle (EV) charging problem with data bias due to COVID-19. 
\end{abstract}

\keywords{Black-box policy, stability, nonlinear control, EV charging affected by COVID-19}



 










\section{Introduction}

Deep neural network (DNN)-based control methods such as deep reinforcement learning/imitation learning have attracted great interests due to the success on a wide range of control tasks such as humanoid locomotion~\cite{schulman2017proximal}, playing Atari~\cite{mnih2013playing} and 3D racing games~\cite{ross2011reduction}. 
These methods are typically model-free and are capable of learning policies and value functions for complex control tasks directly from raw data. In real-world applications such as autonomous driving, it is impractical to dynamically update the already-deployed policy. In those cases, pre-trained black-box policies are applied. Those partially-optimized solutions on the one hand can sometimes be optimal or near-optimal, but on the other hand can be arbitrarily poor in cases where there is unexpected environmental behavior due to, e.g., sample inefficiency~\cite{botvinick2019reinforcement}, reward sparsity~\cite{riedmiller2018learning}, mode collapse~\cite{jabri2019unsupervised}, high variability of policy gradient~\cite{cheng2019control,recht2019tour}, or biased training data~\cite{bai2019model}. This uncertainty raises significant concerns about applications of these tools in safety-critical settings.  Meanwhile, for many real-world control problems, crude information about system models exists, e.g, linear approximations of their state transition dynamics~\cite{jin2020stability,qu2021exploiting}.  Such information can be useful in providing model-based advice to the machine-learned policies.

To represent such situations, in this paper we consider the following infinite-horizon dynamical system consisting of a \textit{known} affine part, used for model-based advice, and an \textit{unknown} nonlinear residual function, which is (implicitly) used in developing  machine-learned (DNN-based) policies:
\begin{align}
\label{eq:nonlinear_dynamic}
x_{t+1} =& \underbrace{A x_t + B u_t}_{\text{\textit{Known} affine part}} + \underbrace{f_t(x_t,u_t),}_{\text{\textit{Unknown} nonlinear residual}} \text{ for } t = 0,\ldots,\infty,
\end{align}
where $x_t\in\mathbb{R}^\nn$ and $u_t\in\mathbb{R}^\nm$ are the system state and the action selected by a controller at time $t$;  $A$ and $B$ are coefficient matrices in the affine part of the system dynamics. Besides the linear components, the system also has a state and action-dependent nonlinear residual $f_t:\mathbb{R}^{\nn}\times\mathbb{R}^{\nm}\rightarrow\mathbb{R}^\nn$ at each time $t\geq 0$, representing the modelling error. The matrices $A$ and $B$ are fixed and known coefficients in the linear approximation of the true dynamics~\eqref{eq:nonlinear_dynamic}.

Given the affine part of the dynamics~\eqref{eq:nonlinear_dynamic}, it is possible to construct a model-based feedback controller $\overline{\pi}$, e.g., a linear quadratic regulator or an $\mathcal{H}_{\infty}$ controller. Compared with a DNN-based policy $\widehat{\pi}(x_t)$, the induced linear controller often has a worse performance on average due to the model bias, but becomes more stable in an adversarial setting. In other words, a DNN-based policy can be as good as an optimal policy in domains where accurate training data has been collected, but can perform \textit{sub-optimally} in other situations; while a policy based on a linearized system is \textit{stabilizing}, so that it has a guaranteed worst-case performance with bounded system perturbations, but can lose out on performance to DNN-based policies in non-adversarial situations. We illustrate this trade-off for the CartPole problem (see Example~\ref{example:cartpole} in Appendix~\ref{example:cartpole}) in Figure~\ref{fig:fixed_lambda}. The figure shows a pre-trained TRPO~\cite{schulman2015trust,stable-baselines3} agent and an ARS~\cite{mania2018simple} agent achieve lower costs when the initial angle of the pole is small; but become less stable when the initial angle increases. Using the affine part of the nonlinear dynamics, a linear quadratic regulator achieves better performance when the initial angle becomes large. Motivated by this trade-off, we ask the following question in this paper:

\textit{Can we equip a sub-optimal machine-learned policy $\widehat{\pi}$ with stability guarantees for the nonlinear system~\eqref{eq:nonlinear_dynamic} by utilizing model-based advice from the known, affine part?}

\begin{wrapfigure}[19]{r}{0.5\textwidth}
    \centering
\includegraphics[scale=0.50]{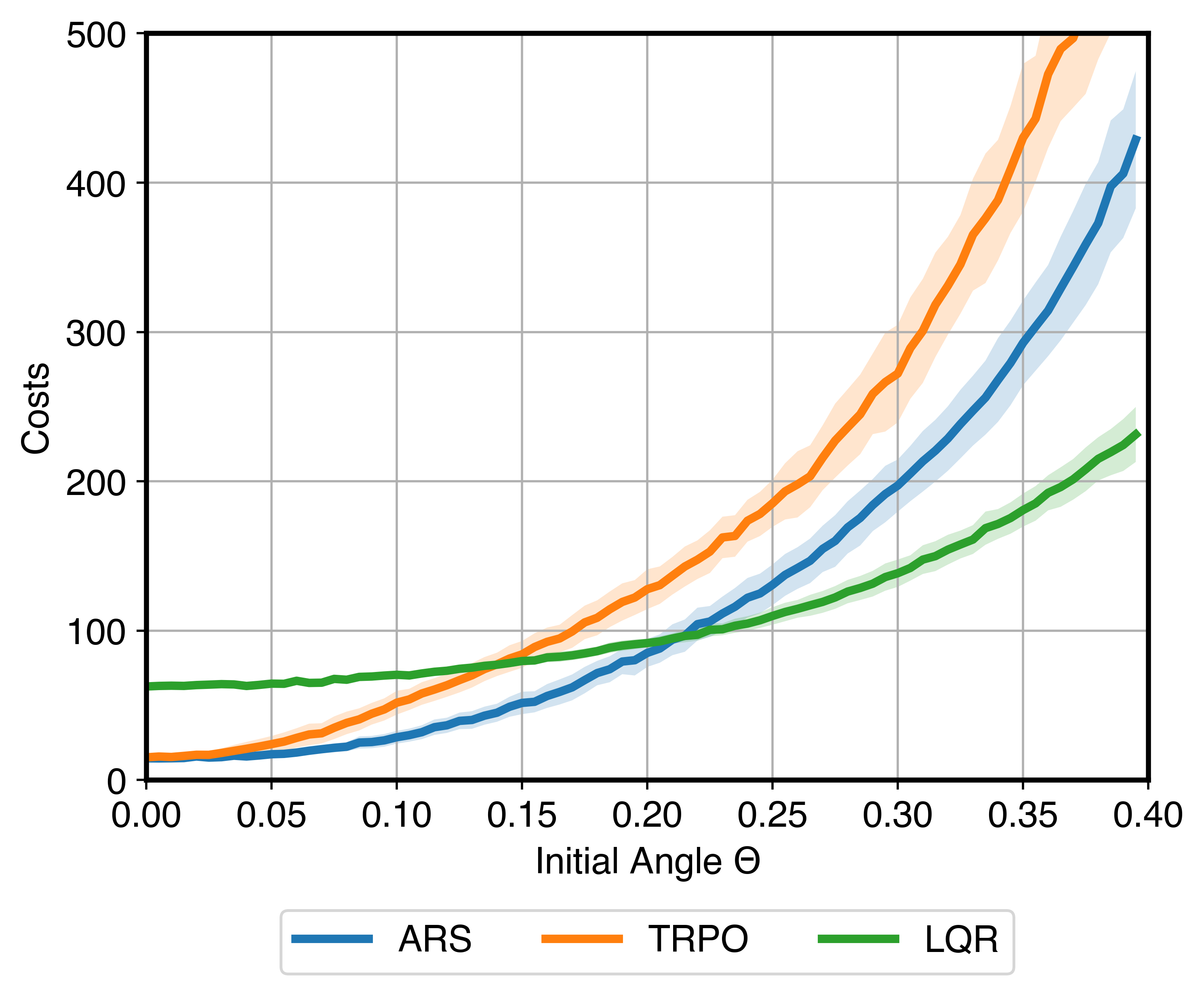}
\caption{Costs of pre-trained TRPO and ARS agents and an LQR when the initial pole angle $\theta$ (unit: radians) varies.}
\label{fig:fixed_lambda}
\end{wrapfigure}

Traditionally, switching between different control policies has been investigated for linear systems~\cite{hespanha2002switching,niemann2004switching}. All candidate polices need to be linear and therefore can be represented by their Youla–Kucera parametrizations (or Q-parameterizations). When a policy is a black-box machine-learned policy modeled by a DNN that may be nonlinear, how to combine or switch between the policies remains an open problem that is made challenging by the fact that the model-free policy typically has no theoretical guarantees associated with its performance. On the one hand, a model-free policy works well on average but, on the other hand, model-based advice  stabilizes the system in extreme cases.

\textbf{Contributions.}  In this work we propose a novel, adaptive policy that combines model-based advice with a black-box machine-learned controller to guarantee stability while retaining the performance of the machine-learned controller when it performs well.  In particular, we consider a nonlinear control problem whose dynamics is given in~\eqref{eq:nonlinear_dynamic}, where we emphasize that the unknown nonlinear residual function $f_t(x_t,u_t)$ is time-varying and depends not only on the state $x_t$ but also the action $u_t$ at each time $t$. Our first result is a negative result (Theorem~\ref{thm:negativity}) showing that a naive convex combination of a black-box model-free policy with model-based advice can lead to instability, even if both policies are stabilizing individually. This negative result highlights the challenges associated with combining model-based and model-free approaches.

%
 
Next, we present a general policy that adaptively combines a model-free black-box policy with model-based advice (Algorithm~\ref{alg:adaptive_policy}). We assume that the model-free policy has some consistency error $\varepsilon$, compared with the optimal policy and that the residual functions $(f_t:t\geq 0)$ are Lipschitz continuous with a Lipschitz constant $C_{\ell}>0$. Instead of employing a hard-switching between policies, we introduce a time-varying \textit{confidence coefficient} $\lambda_t$ that only decays and switches a black-box model-free policy into a stabilizing model-based policy in a smooth way during operation as needed to ensure stabilization. The sequence of confidence coefficients converges to $\lambda \in [0,1]$. Our main result is the following theorem, which establishes a trade-off between competitiveness (Theorem~\ref{thm:competitive}) and stability (Theorem~\ref{thm:stability}) of this adaptive algorithm.

\begin{nonumbertheorem}[Informal] 
\label{thm:main_informal}
With system assumptions (Assumption~\ref{assumption:continuity},~\ref{assumption:stability} and an upper bound on $C_{\ell}$), the adaptive $\lambda$-confident policy  (Algorithm~\ref{alg:adaptive_policy}) has the following properties: (a) the policy is exponentially stabilizing whose decay rate increases when ${{\lambda}}$ decreases; and (b) when the consistency error $\varepsilon$ is small, the competitive ratio of the  policy  satisfies
    \begin{align}
    \label{eq:informal_bound}
      \mathsf{CR}(\varepsilon) =  \left(1-{{\lambda}}\right)\times \underbrace{O(\overline{\mathsf{CR}}_{\mathrm{model}})}_{\textrm{Model-based bound}}+\underbrace{O\left({1}/{{(1-O(\varepsilon))}}\right)}_{\textrm{Model-free error}} + \underbrace{O(C_{\ell}\|x_0\|_{2})}_{\textrm{Nonlinear dynamics error}}. 
    \end{align}
\end{nonumbertheorem}

The theorem shows that the  adaptive $\lambda$-confident policy is guaranteed to be stable. Furthermore, if the black-box policy is close to an optimal control policy (in the sense that the consistency error $\varepsilon$ is small), then the adaptive $\lambda$-confident policy has a bounded competitive ratio that consists of three components. The first one is a bound inherited from a model-based policy; the second term depends on the sub-optimality gap between a black-box policy and an optimal policy; and the last term encapsulates the loss induced by switching from a policy to another and it scales with the $\ell_2$ norm of an initial state $x_0$ and the nonlinear residuals (depending on the Lipschitz constant $C_{\ell}$).

Our results imply an interesting trade-off between stability and sub-optimality, in the sense that if $\lambda$ is smaller, it is guaranteed to stabilize with a higher rate and if $\lambda$ becomes larger, it is able to have a smaller competitive ratio bound when provided with a high-quality black-box policy.  Different from the linear case, where a cost characterization lemma can be directly applied to bound the difference between the policy costs and optimal costs in terms of the difference between their actions~\cite{li2022robustness}, for the case of nonlinear dynamics~\eqref{eq:nonlinear_dynamic}, we introduce an auxiliary linear problem to derive an upper bound on the dynamic regret, whose value can be decomposed into a quadratic term and a term induced by the nonlinearity. The first term can be bounded via a generalized characterization lemma and becomes the \textit{model-based bound} and \textit{model-free error} in~\eqref{eq:informal_bound}. The second term becomes a \textit{nonlinear dynamics error} via a novel sensitivity analysis of an optimal nonlinear policy based on its Bellman equation. Finally, we use the CartPole problem to demonstrate the efficacy of the adaptive $\lambda$-confident policy.


\textbf{Related work.} Our work is related to a variety of classical and learning-based policies for control and reinforcement learning (RL) problems that focus on combining model-based and model-free approaches.

\textit{Combination of model-based information with model-free methods.} 
Our paper adds to the recent literature seeking to combine model-free and model-based policies for online control. Some prominent recent papers with this goal include the following.  First, in~\cite{PonGuDalLev18}, Q-learning is connected with model predictive control (MPC) whose constraints are Q-functions encapsulating the state transition information. Second, MPC methods with penalty terms learned by model-free algorithms are considered in~\cite{rosolia2017learning}. Third, deep neural network dynamics models are used to initialize a model-free learner to improve the sample efficiency while maintaining the high task-specific performance~\cite{nagabandi2018neural}. Next, using this idea, the authors of~\cite{qu2021exploiting} consider a more concrete dynamical system $x_{t+1}=Ax_t+B u_t +f(x_t)$ (similar to the dynamics in~\eqref{eq:nonlinear_dynamic} considered in this work) where $f$ is a state-dependent function and they show that a model-based initialization of a model-free policy is guaranteed to converge to a near-optimal linear controller. Another approach uses an $\mathcal{H}_{\infty}$ controller integrated into model-free RL algorithms for variance reduction~\cite{cheng2019control}. Finally, the model-based value expansion is proposed in~\cite{feinberg2018model} as a method to incorporate learned dynamics models in model-free RL algorithms.  Broadly, despite many heuristic combinations of model-free and model-based policies demonstrating empirical improvements, there are few theoretical results explaining and verifying the success of the combination of model-free and model-based methods for control tasks. Our work contributes to this goal.

\textit{Combining stabilizing linear controllers.}  The proposed algorithm in this work combines existing controllers and so is related to the literature of combining stabilizing linear controllers.  
A prominent work in this area is~\cite{hespanha2002switching}, which shows that with proper controller realizations, switching between a family of stabilizing controllers uniformly exponentially stabilizes a linear  time-invariant (LTI) system. Similar results are given in~\cite{niemann2004switching}.
The techniques applied in~\cite{niemann2004switching,hespanha2002switching} use the fact that all the stabilizing controllers can be expressed using the Youla parameterization. 
Different from the classical results of switching between or combining stabilizing controllers, in this work, we generalize the idea to the combination of a linear model-based policy and a model-free policy, that can be either linear or nonlinear.

\textit{Learning-augmented online problems.}  Recently, the idea of augmenting robust/competitive online algorithms with machine-learned advice has attracted attention in online problems in settings like online caching \cite{lykouris2018competitive}, ski-rental \cite{purohit2018improving,wei2020optimal}, smoothed online convex optimization~\cite{rutten2022online,christianson2022chasing} and linear quadratic control~\cite{li2022robustness}. In many of these learning-augmented online algorithms, a convex combination of machine-learned (untrusted) predictions and robust decisions is involved. For instance, in~\cite{li2022robustness}, competitive ratio upper bounds of a $\lambda$-confident policy are given for a linear quadratic control problem. The policy $\lambda\pi_{\mathrm{MPC}}+(1-\lambda)\pi_{\mathrm{LQR}}$ combines linearly a linear quadratic regulator $\pi_{\mathrm{LQR}}$ and an MPC policy $\pi_{\mathrm{MPC}}$ with machine-learned predictions where $\lambda\in [0,1]$ measures the confidence of the machine-learned predictions. To this point, no results on learning-augmented controllers for nonlinear control exist.  In this work, we focus on the case of nonlinear dynamics and show a general negativity result (Theorem~\ref{thm:negativity}) such that a simple convex combination between two policies can lead to unstable outputs and then proceed to provide a new approach that yields positive results.

\begin{table}[h]
    \footnotesize
\renewcommand{\arraystretch}{1.3}
    \centering
    \begin{tabular}{l|c|c|c|l|c}
\specialrule{.13em}{.1em}{.1em} 
    & \textbf{Setting}
    &\textbf{Known} & \textbf{Unknown}  & (Partial) \textbf{Assumption}(s) & \textbf{Objective}\\
    \hline
   \cite{berkenkamp2017safe} &  Episodic & $h$ & $f$ & $h,f\in \mathcal{C}^{0,1}$ & Safe exploration\\
   \cite{cheng2019control} & Episodic & $f^{known}$ & $f^{unknown}$ &  \tabincell{l}{$\|\widehat{\pi}-\overline{\pi}\|_2\leq C_{\pi}$ \\Stabilizable $f^{known}$}  & Lyapunov stability \\
\cite{jin2020stability} & Episodic & $A,B$ & $g_t(x)$ & Hurwitz $A$ & Input-ouput stability\\
  \cite{qu2021exploiting} & Episodic &  $A,B$ & $f(x)$ & $f\in \mathcal{C}^{0,1}$, Stabilizable $A,B$ & Lyapunov stability \\
   \cite{chow2018lyapunov,perkins2002lyapunov,zanon2020safe} & Episodic  &  & (C)MDP & Feasible baseline~\cite{chow2018lyapunov} & Safety and stability \\
  This work & $1$-trajectory & $A,B$ & $f_t(x,u)$ & $f_t\in \mathcal{C}^{0,1}$, Stabilizable $A,B$ & Stability, $\mathsf{CR}$ bound \\
\specialrule{.13em}{.1em}{.1em} 
    \end{tabular}
  \label{table:summary}
\end{table}

\textit{Stability-certified RL.} 
Another highly related line of work is the recent research on developing safe RL with stability guarantees. In~\cite{berkenkamp2017safe}, Lyapunov analysis is applied to guarantee the stability of a model-based RL policy. If an $\mathcal{H}_{\infty}$ controller $\widehat{\pi}_{\mathcal{H}_{\infty}}$ is close enough to a model-free deep RL policy $\overline{\pi}_{\mathrm{RL}}$, by combining the two policies linearly $\lambda\widehat{\pi}_{\mathcal{H}_{\infty}}+(1-\lambda)\overline{\pi}_{\mathrm{RL}}$  at each time in each training episode, asymptotic
stability and forward invariance can be guaranteed using Lyapunov analysis but the convergence rate is not provided~\cite{cheng2019control}. In practice,~\cite{cheng2019control} uses an empirical approach to choose a time-varying factor $\lambda$ according to the temporal difference error.  Robust model predictive control (MPC) is combined with deep RL to ensure safety and stability~\cite{zanon2020safe}. Using regulated policy gradient, input-output stability is guaranteed for a continuous nonlinear control model $f_t(x(t))=Ax(t) + Bu(t)+g_t(x(t))$~\cite{jin2020stability}. In those works, a common assumption needs to be made is the ability to access and update the deep RL policy during the episodic training steps. Moreover, in the state-of-the-art results, the stability guarantees are proven, either considering an aforementioned episodic setting when the black-box policy can be improved or customized~\cite{berkenkamp2017safe,jin2020stability}, or assuming a small and bounded output distance between a black-box policy and a stabilizing policy for any input states to construct a Lyapunov equation~\cite{cheng2019control}, which is less realistic. Stability guarantees under different model assumptions such as (constrained) MDPs have been studied~\cite{chow2018lyapunov,perkins2002lyapunov,zanon2020safe}. Different from the existing literature, the result presented in this work is unique and novel in the sense that we consider stability and sub-optimality guarantee for black-box deep policies in a single trajectory such that we can neither learn from the environments nor update the deep RL policy through extensive training steps. Denote by $\mathcal{C}^{0,1}$ the class of Lipschiz continuous functions (with domains, ranges and norms specified according to the contexts), the related results are summarized in the table above.


\section{Background and Model}

\label{sec:problem}

We consider the following infinite-horizon quadratic control problem with nonlinear dynamics:

\begin{align}
\label{eq:system}
\min_{(u_t: t\geq 0)} \sum_{t=0}^{\infty}& x_t^\top Q x_t + u_t^\top R u_t, \
\textrm{subject} \ \textrm{to} \ \eqref{eq:nonlinear_dynamic}
\end{align}
where in the problem $Q, R\succ 0$ are $\nn\times\nn$ and $m \times m$ positive definite matrices and each $f_t:\mathbb{R}^{\nn}\times\mathbb{R}^{\nm}\rightarrow\mathbb{R}^{\nn}$ in~\eqref{eq:nonlinear_dynamic} is an unknown nonlinear function representing state and action-dependent perturbations. An initial state $x_0$ is fixed. We use the following assumptions throughout this paper. Our first assumption is the Lipschitz continuity assumption on the residual functions and it is standard~\cite{qu2021exploiting}.
Note that $\|\cdot\|$ denotes the Euclidean norm throughout the paper.

\begin{assumption}[Lipschitz continuity]
\label{assumption:continuity}
The function $f_t:\mathbb{R}^{\nn}\times\mathbb{R}^{\nm}\rightarrow\mathbb{R}^{\nn}$ is Lipschitz continuous for any $t\geq 0$, i.e., there is a constant $C_{\ell}\geq 0$ such that $\|f_t(x)-f_t(y)\|\leq C_{\ell}\|x-y\|$ for any $x,y\in\mathbb{R}^\nn$ and $t\geq 0$. Moreover, $f(\mathbf{0})=0$.
\end{assumption}

Next, we make a standard assumption on the system stability and cost function~\cite{mania2019certainty,dean2020sample}.

\begin{assumption}[System stabilizability and costs]
\label{assumption:stability}
The pair of matrices $(A,B)$ is stabilizable, i.e., there exists a real matrix $K$ such that the spectral radius $\rho(A-BK)<1$. We assume $Q,R\succeq \sigma I$. Furthermore, denote $\kappa\coloneqq{\max\{2,\|A\|,\|B\|\}}$.
\end{assumption}

In summary, our control agent is provided with a black-box policy $\widehat{\pi}$ and system parameters $A,B,Q,R$. The goal is to utilize $\widehat{\pi}$ and system information to minimize the quadratic costs in~\eqref{eq:system}, without knowing nonlinear residuals $(f_t:t\geq 0)$. Next, we present our policy assumptions.


\textbf{Model-based advice.} 
In many real-world applications, linear approximations of the true nonlinear system dynamics are known, i.e., the known affine part of~\eqref{eq:nonlinear_dynamic} is available to construct a stabilizing policy $\overline{\pi}$.
To construct $\overline{\pi}$, the assumption that $(A,B)$ is stabilizable implies that the following discrete algebraic Riccati equation (DARE) has a unique semi-positive definite solution $P$ that stabilizes the closed-loop system~\cite{wonham1968matrix}:
\begin{align}
\label{eq:dare}
    P=Q+A^\top P A - A^\top PB (R+B^\top P B)^{-1} B^\top PA.
\end{align}

Given $P$, define $K \coloneqq (R+B^\top P B)^{-1} B^\top P A$. The closed-loop system matrix $F\coloneqq A-BK$ must have a spectral radius $\rho(F)$ less than $1$. Therefore, the Gelfand’s formula implies that there must exist a constant $C_F>0$, $\rho\in(0,1)$ such  that $\| F^t\|\leq C_F\rho^t$, for any $t\geq 0$. The model-based advice considered in this work is then defined as a sequence of actions $(u_t:t\geq 0)$ provided by a linear quadratic regulator (LQR) such that $u_t=\overline{\pi}(x_t)=-Kx_t$.

\textbf{Black-box model-free policy.}
To solve the nonlinear control problem in~\eqref{eq:system}, we take advantage of both model-free and model-based approaches. We assume a pre-trained model-free policy, whose policy is denoted by $\widehat{\pi}:\mathbb{R}^{\nn}\rightarrow\mathbb{R}^{\nm}$, is provided beforehand. The model-free policy is regarded as a ``black box'',  whose detail is not the major focus in this paper. The only way we interact with it is to obtain a suggested action $\widehat{u}_t=\widehat{\pi}(x_t)$ when feeding into it the current system state $x_t$. The performance of the model-free policy is not guaranteed and it can make some error, characterized by the following definition, which compares $\widehat{\pi}$ against a clairvoyant optimal controller $\pi_t^*$ knowing the nonlinear residual perturbations in hindsight:

\begin{definition}[$\varepsilon$-\textit{consistency}]
\label{def:epsilon_consistency}
A policy $\pi:\mathbb{R}^{\nn}\rightarrow\mathbb{R}^{\nm}$ is called \textbf{$\varepsilon$-\textit{consistent}} if there exists $\varepsilon>0$ such that for any $x\in\mathbb{R}^{\nn}$ and $t\geq 0$, $\left\|{\pi}(x) - {\pi}_t^*(x)\right\|\leq \varepsilon\|x\|$ where $\pi_t^*$ denotes an optimal policy at time $t$ knowing all the nonlinear residual perturbations $(f_t:t\geq 0)$ in hindsight and $\varepsilon$ is called a \textbf{\textit{consistency error}}.
\end{definition}

The parameter $\varepsilon$ measures the difference between the action given by the oracle policy $\widehat{\pi}$ and the optimal action given the state $x$. There is no guarantee that  $\varepsilon$ is small. With  prior knowledge of the nonlinearity of system gained from data, the sub-optimal model-free policy $\widehat{\pi}$ suffers a \textit{consistency error} $\varepsilon>0$, which can be either small if the black-box policy is trained by unbiased data; or high because of the high variability issue for policy gradient deep RL algorithms ~\cite{cheng2019control,recht2019tour} and distribution shifts of the environments. In these cases, the error $\varepsilon>0$ can be large. In this paper, we augment a black-box model-free policy with stability guarantees using the idea of adaptively switching it to a model-based stabilizing policy $\overline{\pi}$, which often exists provided with exact or estimates of system parameters $A,B,Q$ and $R$. The linear stabilizing policy is conservative and highly sub-optimal as it is neither designed based on the exact nonlinear model nor interacts with the environment like the training of $\widehat{\pi}$ potentially does.

\textbf{Performance metrics.}  Our goal is to ensure stabilization of states while also providing good performance, as measured by the competitive ratio.  Formally, a  policy $\pi$ is (asymptotically) \textit{stabilizing} if it induces a sequence of states $(x_t:t\geq0)$ such that $\|x_t\|\rightarrow 0$ as $t\rightarrow \infty$. If there exist $C>0$ and $0\leq\gamma<1$ such that $\|x_t\|\leq C\gamma^{t}\|x_0\|$ for any $t\geq 0$, the corresponding policy is said to be exponentially stabilizing. To define the competitive ratio, let $\mathsf{OPT}$ be the offline optimal cost of~\eqref{eq:system} induced by optimal control policies $(\pi_t^*:t\geq 0)$ when the nonlinear residual functions $(f_t:t\geq 0)$ are known in hindsight, and $\mathsf{ALG}$ be the cost achieved by an online policy. Throughout this paper we assume $\mathsf{OPT}>0$. We formally define the competitive ratio as follows.

\begin{definition} 
Given a policy, the corresponding \textbf{competitive ratio}, denoted by $\mathsf{CR}$, is defined as the smallest constant $C\geq 1$ such that $\mathsf{ALG}\leq C \cdot \mathsf{OPT}$ for fixed $A,B,Q,R$ satisfying Assumption~\ref{assumption:stability} and any adversarially chosen residual functions $(f_t:t\geq 0)$ satisfying Assumption~\ref{assumption:continuity}.
\end{definition} 


\section{Warm-up: A Naive Convex Combination}
\label{sec:negativity}

The main results in this work focus on augmenting a black-box policy $\widehat{\pi}$ with stability guarantees while minimizing the quadratic costs in~\eqref{eq:system}, provided with linear system parameters $A,B,Q,R$ of a nonlinear system. Before proceeding to our policy, to highlight the challenge of combining model-based advice with model-free policies in this setting we first consider a simple strategy for combining the two via a convex combination.  This is an approach that has been proposed and studied previously, e.g.,~\cite{cheng2019control,li2022robustness}.  However, we show that it can be problematic in that it can yield an unstable policy even when the two policies are stabilizing individually.  Then, in Section~\ref{sec:main}, we propose an approach that overcomes this challenge.

A natural approach for incorporating model-based advice is a convex combination of a model-based control policy $\overline{\pi}$ and a black-box model-free policy $\widehat{\pi}$. The combined policy generates an action $u_t = \lambda \widehat{\pi}(x_t) + (1-\lambda) \overline{\pi}(x_t)$ given a state $x_t$ at each time, where $ \lambda\in [0,1]$. The coefficient $\lambda$ determines a confidence level such that if $\lambda$ is larger, we trust the black-box policy more and vice versa.  
In the following, however, we highlight that, in general, the convex combination of two polices can yield an unstable policy, even if the two policies are stabilizing, with a proof in Appendix~\ref{app:proof_negativity}.

\begin{theorem}
\label{thm:negativity}
Assume $B$ is an $\nn\times\nn$ full-rank matrix with $\nn>1$. For any $\lambda \in (0,1)$ and any linear controller $K_1$ satisfying $A-BK_1\neq 0$, there exists a linear controller $K_2$ that stabilizes the system such that their convex combination $\lambda K_2 + (1-\lambda) K_1$ is unstable, i.e., the spectral radius $\rho(A-B(\lambda K_2 + (1-\lambda)K_1))>1$.
\end{theorem}

Theorem~\ref{thm:negativity} brings up an issue with the strategy of combining a stabilizing policy with a model-free policy. Even if both the model-based and model-free policies are stabilizing, the combined controller can lead to unstable state outputs. In general, the space of stabilizing linear controllers $\{K\in\mathbb{R}^{\nn\times\nm}: K \ \mathrm{is} \  \mathrm{stabilizing}\}$ is nonconvex~\cite{zheng2020equivalence}. The result in~Theorem~\ref{thm:negativity} is a stronger statement. It implies that for any arbitrarily chosen linear policy $K_1$ and a coefficient $\lambda\in (0,1)$,
we can always adversarially select a second policy $K_2$ such that their convex combination leads to an unstable system. It is worth emphasizing that the second policy does not necessarily have to be a complicated nonlinear policy. Indeed, in our proof, we construct a linear policy $K_2$ to derive the conclusion. In our problem, the second policy $K_2$ is assumed to be a black-box policy $\widehat{\pi}$ potentially parameterized by a deep neural network, yielding much more uncertainty on a similar convex combination. As a result, we must be careful when combining policies together.

Note that the idea of applying a convex combination of an RL policy and a control-theoretic policy linearly is not a new approach and similar policy combinations have been proposed in previous studies~\cite{cheng2019control,li2022robustness}.
However, in those results, either the model-free policy is required to satisfy specific structures~\cite{li2022robustness} or to be close enough to the stabilizing policy~\cite{cheng2019control} to be combined. In~\cite{li2022robustness}, a learning-augmented policy is combined with a linear quadratic regulator, but the learning-augmented policy has a specific form and it is not a black-box policy. In~\cite{cheng2019control}, a deep RL policy $\widehat{\pi}$ is combined with an $\mathcal{H}^{\infty} $ controller $\overline{\pi}$ and they need to satisfy that for any state $x\in\mathbb{R}^{\nn}$, $\|\widehat{\pi}(x)-\overline{\pi}(x)\|\leq C_{\pi}$ for some $C_{\pi}>0$. However, it is possible that when the state norm $\|x\|$ becomes large, the two policies in practice behave entirely differently. Moreover, it is hard to justify the benefit of combining two policies, conditioned on the fact that they are already similar.  Given that those assumptions are often not satisfied or hard to be verified in practice, we need another approach to guarantee worst-case stability when the black-box policy is biased and in addition ensure sub-optimality if the black-box policy works well.

\section{Adaptive $\lambda$-confident Control}\label{sec:main}


%

Motivated by the challenge highlighted in the previous section, we now propose a general framework that adaptively selects a sequence of monotonically decreasing confidence coefficients $(\lambda_t:t\geq 0)$ in order to switch between black-box and stabilizing model-based policies. We show that it is possible to guarantee a bounded competitive ratio when the black-box policy works well, i.e., it has a small consistency error $\varepsilon$, and guarantee stability in cases when the black-box policy performs poorly.


\begin{wrapfigure}[18]{L}{0.44\textwidth}
\small
\begin{algorithm}[H]
\KwData{System parameters $A,B,Q,R$, $\alpha$}
\DontPrintSemicolon
\For{$t\geq 0$}{

\lIf{$t=0$}{Initialize $\lambda_0\longleftarrow 1$}
\lIf{$\|x_t\|=0$}{$\lambda_t\longleftarrow \lambda_t$}
\Else{Obtain a coefficient $\lambda'$ 

\Comment{{\scriptsize\textsf{\textit{Online learning (Eq.~\eqref{eq:linear_lambda_t})}}}}

\lIf{$\lambda'>0 $ and $\lambda_{t-1}>\alpha$}{$\lambda_{t} \longleftarrow \min\{ \lambda', \lambda_{t-1}-\alpha\}$}
\lElse{
$\lambda_t \longleftarrow 0$
}
}
Generate an action $u_t=\lambda_t\widehat{\pi}(x_t)+(1-\lambda_t)\overline{\pi}(x_t)$

Update state according to~\eqref{eq:nonlinear_dynamic}
}
\caption{Adaptive $\lambda$-confident}
\label{alg:adaptive_policy}
\end{algorithm}
\end{wrapfigure}

The adaptive $\lambda$-confident policy introduced in~Algorithm~\ref{alg:adaptive_policy} involves an input coefficient $\lambda'$ at each time. The value of $\lambda_t$ can either be $\lambda_{t-1}$ decreased by a fixed step size $\alpha$, or a variable learned from known system parameters in~\eqref{eq:system} combined with observations of previous states and actions. In Section~\ref{sec:online_learning}, we consider an online learning approach to generate a value of $\lambda'$ at each time $t$, but it is worth emphasizing that the adaptive policy in~Algorithm~\ref{alg:adaptive_policy} and its theoretical guarantees in Section~\ref{sec:theoretical_guarantees} do not require specifying a detailed construction of $\lambda'$. 

The adaptive  policy differs from the naive convex combination that has been discussed in Section~\ref{sec:negativity} in that it adopts a sequence of time-varying and monotonically decreasing coefficients $(\lambda_t:t\geq 0)$ to combine a black-box policy and a model-based stabilizing policy, where the former policy is adaptively switched to the later one. The coefficient $\lambda_t$ converges to $\lim_{t\rightarrow\infty}\lambda_t={{\lambda}}$, where the limit ${{\lambda}}$ can be a positive value, if the state converges to a target equilibrium ($\mathbf{0}$ under our model assumptions) before $\lambda_t$ decreases to zero. This helps stabilize the system under assumptions on the Lipschitz constant $C_{\ell}$ of unknown nonlinear residual functions and if the black-box policy is near-optimal, a bounded competitive ratio is guaranteed, as we show in the next section.

\subsection{Theoretical guarantees}
\label{sec:theoretical_guarantees}

The theoretical guarantees we obtain are two-fold. First, we show that the adaptive $\lambda$-confident policy in Algorithm~\ref{alg:adaptive_policy} is stabilizing, as stated in Theorem~\ref{thm:stability}. Second, in addition to stability, we show that the policy has a bounded competitive ratio, if the black-box policy used has a small consistency error (Theorem~\ref{thm:competitive}). Note that if a black-box policy has a large consistency error $\varepsilon$, without using model-based advice, it can lead to instability and therefore possibly an unbounded competitive ratio.

\textbf{Stability.}
Before presenting our results, we introduce some new notation for convenience. Denote by $t_0$ the smallest time index when $\lambda_t = 0$ or $x_t=\mathbf{0}$ and note that $\mathbf{0}$ is an equilibrium state. Denote by ${{\lambda}}=\lim_{t\rightarrow\infty}\lambda_t$. Since $(\lambda_t:t\geq 0)$ is a monotonically decreasing sequence and $\lambda_t$ has a lower bound, $t_0$ and ${{\lambda}}$ exist and are unique. Let $H\coloneqq  R+B^{\top}PB$.
Define the parameters
$
\gamma\coloneqq \rho+C_F C_{\ell}(1+\|K\|)$ and
$\mu \coloneqq C_F\left(\varepsilon\left(C_{\ell}+\|B\|\right)+C_{a}^{\mathsf{sys}} C_{\ell}\right)$
where $A,B,Q,R$ are the known system parameters in~\eqref{eq:system}, $P$ has been defined in the Riccati equation~\eqref{eq:dare}; $C_F>0$ and $\rho\in(0,1)$ are constants such that $\| F^t\|\leq C_F\rho^t$, for any $t\geq 0$ as defined in Section~\ref{sec:problem}; $C_{\ell}$ is the Lipschitz constant in Assumption~\ref{assumption:continuity}; $\varepsilon>0$ is the consistency error in Definition~\ref{def:epsilon_consistency}; Finally, $C_{a}^{\mathsf{sys}}, C_{b}^{\mathsf{sys}},C_{c}^{\mathsf{sys}}>0$ are constants that only depend on the known system parameters in~\eqref{eq:system} and they are listed in Appendix~\ref{app:notation}.

%

Given the above notation, the theorem below guarantees stability of the adaptive $\lambda$-confident policy.

\begin{theorem}
\label{thm:stability}
Suppose the Lipschitz constant $C_{\ell}$ satisfies 
$C_{\ell}<\frac{1-\rho}{{C_F}(1+\|K\|)}.
$ The adaptive $\lambda$-confident policy  (Algorithm~\ref{alg:adaptive_policy}) is an exponentially stabilizing policy such that $ \|x_t\| = O\left((\mu/\gamma)^{t_0}\gamma^t\right)\|x_0\|$.

\end{theorem}

For Theorem~\ref{thm:stability} to hold such that $\gamma<1$, the Lipschitz constant $C_{\ell}$ needs to have an upper bound $\smash{\frac{1-\rho}{{C_F}(1+\|K\|)}}$ where $C_F>0$ and $\rho\in(0,1)$ are constants such that $\| F^t\|\leq C_F\rho^t$, for any $t\geq 0$. Since $A$ and $B$ are stabilizable, such $C_F$ and $\rho$ exist. The upper bound 
only depends on the known system parameters $A,B,Q$ and $R$. A small enough Lipschitz constant is required to guarantee stability. For instance, in~\cite{qu2021exploiting}, convergence exponentially to the equilibrium state is guaranteed when the Lipschitz constant satisfies $\smash{C_{\ell}=O\left(\frac{\sigma^2(1-\rho)^8}{{\kappa}^9{C_F}^{15}}\right)}$. 

\textbf{Competitiveness.}
Define a constant 
$
\smash{\overline{\mathsf{CR}}_{\mathrm{model}}\coloneqq  2\kappa(\frac{{{C_F}}\|P\|}{ 1-\rho})^2/\sigma}
$. The theorem below implies that when the model-free policy error $\varepsilon$ and the Lipschitz constant $C_{\ell}$
of the residual functions are  small enough, Algorithm~\ref{alg:adaptive_policy} is competitive. 

\begin{theorem}
\label{thm:competitive}
Suppose the Lipschitz constant satisfies $C_{\ell}<\min\left\{1, C_{a}^{\mathsf{sys}},  C_{c}^{\mathsf{sys}}\right\}
$.
When the consistency error satisfies $\varepsilon<\min\left\{ \frac{\sigma}{2\|H\|},\frac{1/{C_F}-C^{\mathsf{sys}}_{a}C_{\ell}}{C_{\ell}+\|B\|}\right\}$, the competitive ratio of the adaptive $\lambda$-confident policy  (Algorithm~\ref{alg:adaptive_policy}) is bounded by
    \begin{align*}  \mathsf{CR}(\varepsilon) =  (1-{{\lambda}})\overline{\mathsf{CR}}_{\mathrm{model}}+O\left({1}/\left({{1-\frac{2\|H\|}{\sigma}}\varepsilon}\right)\right) + O(C_{\ell}\|x_0\|). 
    \end{align*}
\end{theorem}

Combining Theorem~\ref{thm:stability} and~\ref{thm:competitive}, our main results are proved when the Lipschitz constant satisfies $C_{\ell}<\min\left\{1, C_{a}^{\mathsf{sys}},  C_{c}^{\mathsf{sys}},\frac{(1-\rho)}{{C_F}(1+\|K\|)}\right\} 
$.  Theorem~\ref{thm:stability} and~\ref{thm:competitive} have some interesting implications. First, if the selected time-varying confidence coefficients converge to a  $\lambda$ that is large, then we trust the black-box policy and use a higher weight in the per-step combination. This requires a slower decaying rate of $\lambda_t$ to zero so as a trade-off, $t_0$ can be higher and this leads to a weaker stability result and vice versa. In contrast, when the nonlinear dynamics in~\eqref{eq:nonlinear_dynamic} becomes linear with unknown constant perturbations,~\cite{li2022robustness} shows a trade-off between robustness and consistency, i.e., a universal competitive ratio bound holds, regardless the error of machine-learned predictions. Different from the linear case where a competitive ratio bound always exists and can be decomposed into terms parameterized by some confidence coefficient $\lambda$, for the nonlinear system dynamics~\eqref{eq:nonlinear_dynamic}, there are additional terms due to the nonlinearity of the system that can only be bounded if the consistency error $\varepsilon$ is small. This highlights a fundamental difference between linear and nonlinear systems, where the latter is known to be more challenging. Proofs of Theorem~\ref{thm:stability} and~\ref{thm:competitive} are provided in Appendix~\ref{app:proof_stability} and~\ref{app:proof_competitiveness}.

\section{Practical Implementation and Experiments}

\subsection{Learning confidence coefficients online}
\label{sec:online_learning}

Our main results in the previous section are stated without specifying a sequence of confidence coefficients $(\lambda_t:t\geq 0)$ for the policy; however in the following we introduce an online learning approach to generate confidence coefficients based on observations of actions, states and known system parameters. The negative result in  Theorem~\ref{thm:negativity} highlights that the adaptive nature of the confidence coefficients in Algorithm \ref{alg:adaptive_policy} are crucial to ensuring stability.  Naturally, learning the values of the confidence coefficients $(\lambda_t:t\geq 0)$ online can further improve performance.

In this section we propose an online learning approach based on a linear parameterization of a black-box model-free policy
$
\widehat{\pi}_t(x) = -K x -H^{-1}B^\top\sum_{\tau=t}^{\infty}\left(F^\top\right)^{t-\tau}P \widehat{f}_{\tau}
$ where $(\widehat{f}_t:t\geq 0)$ are parameters representing estimates of the residual functions for a black-box policy. Note that when $\widehat{f}_t=f_t^*\coloneqq f_t(x_t^*,u_t^*)$ where $x_t^*$ and $u_t^*$ are optimal state and action at time $t$ for an optimal policy, then the model-free policy is optimal. In general, a black-box model-free policy $\widehat{\pi}$ can be nonlinear, the linear parameterization provides an example of how the time-varying confidence coefficients $(\lambda_t:t\geq 0)$ are selected and the idea can be extended to nonlinear parameterizations such as kernel methods.

Under the linear parameterization assumption, for linear dynamics,~\cite{li2022robustness} shows that the optimal choice of $\lambda_{t+1}$ that minimizes the gap between the policy cost and optimal cost for the $t$ time steps is 

\begin{align}
\label{eq:optimal_lambda}
  \lambda_{t+1} =   {\Big(\sum_{s=0}^{t}\left(\eta(f^*;s,t)\right)^{\top} H \left(\eta(\widehat{f};s,t)\right)}\Big)\big/\Big({\sum_{s=0}^{t}\left(\eta(\widehat{f};s,t)\right)^{\top} H \left(\eta(\widehat{f};s,t)\right)}\Big),
\end{align}

where $\eta(f;s,t)\coloneqq \sum_{\tau=s}^{t}\left(F^\top\right)^{\tau-s} P f_\tau$. Compared with a linear quadratic control problem, computing $\lambda_t$ in~\eqref{eq:optimal_lambda} raises two problems. The first is different from a linear dynamical system where true perturbations can be observed, the optimal actions and states are unknown, making the computation of the term $\eta(f^*;s,t-1)$ impossible. The second issue is similar. Since the model-free policy is a black-box, we do not know the parameters $(\widehat{f}_t:t\geq 0)$ exactly. Therefore, we use approximations to compute the terms $\eta\left(f^*;s,t\right)$ and $\eta(\widehat{f};s,t)$ in~\eqref{eq:optimal_lambda}
and the linear parameterization and linear dynamics assumptions are used to derive the approximations  respectively with details provided in Appendix~\ref{app:notation}. Let $\left(BH^{-1}\right)^\dagger$ denote the Moore–Penrose inverse of $BH^{-1}$.  Combining~\eqref{eq:linear_policy} and~\eqref{eq:linear_dynamic} with~\eqref{eq:optimal_lambda} yields the following online-learning choice of a confidence coefficient $\lambda_t=\min\left\{\lambda',\lambda_{t-1} - \alpha\right\}$ where $\alpha>0$ is a fixed step size and

\begin{align}
\label{eq:linear_lambda_t}
   \lambda'\coloneqq  \frac{\sum_{s=1}^{t-1}\left(\sum_{\tau=s}^{t-1}\left(F^\top\right)^{\tau-s}P\left(Ax_{\tau}+Bu_{\tau}-x_{\tau+1}\right)\right)^{\top} B \left(\widehat{u}_s + K x_s\right)}{\sum_{s=0}^{t-1} \left(\widehat{u}_s + K x_s\right)^\top \left(BH^{-1}\right)^\dagger B \left(\widehat{u}_s + K x_s\right)}
\end{align}
based on the crude model information $A,B,Q,R$ and previously observed states, model-free actions and policy actions. This online learning process provides a choice of the confidence coefficient in Algorithm~\ref{alg:adaptive_policy}. It is worth noting that other approaches for generating $\lambda_t$ exist, and our theoretical guarantees apply to any approach.  





\subsection{Experiments}


\begin{figure}[t]
    \centering
    
\includegraphics[scale=0.25]{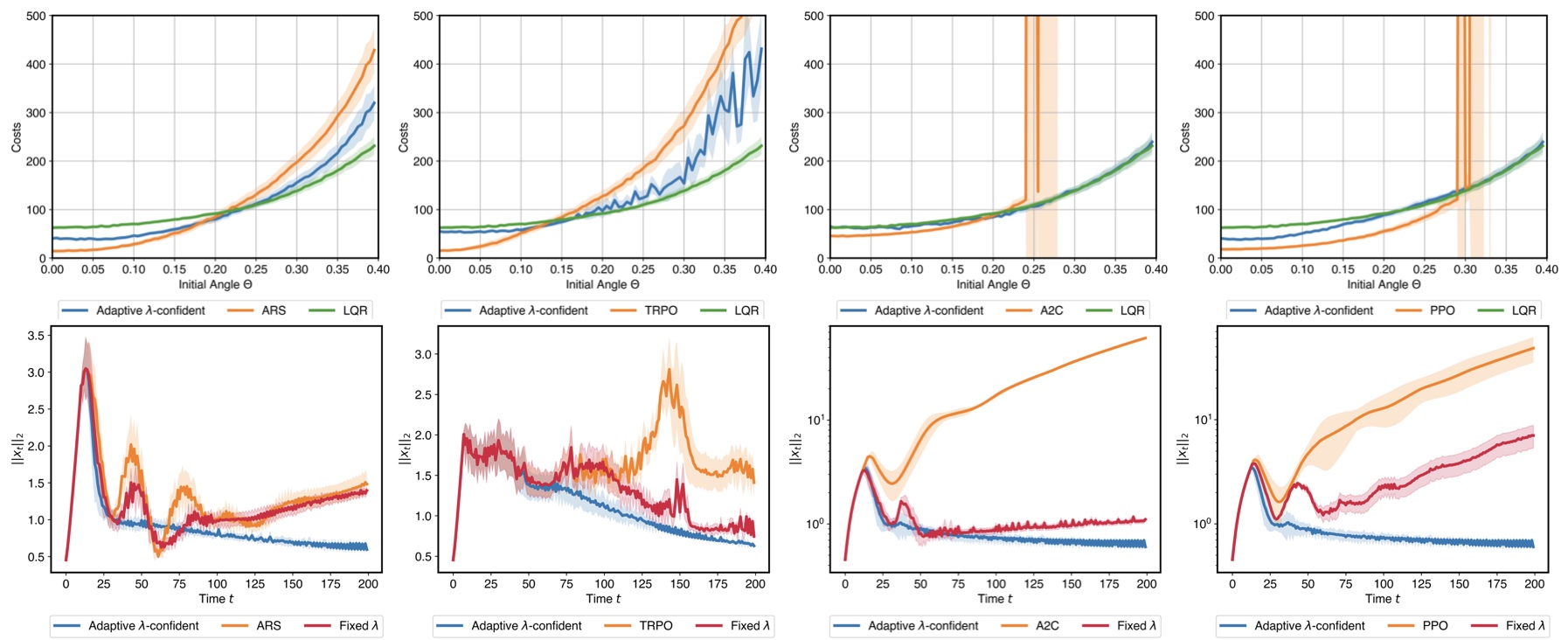}

\caption{Top (\textit{Competitiveness}): costs of pre-trained RL agents, an LQR and the adaptive policy when the initial pole angle $\theta$ (unit: radians) varies. Bottom (\textit{stability}): convergence of $\|x_t\|$ in $t$ with $\theta=0.4$ for pre-trained RL agents, a naive combination (Section~\ref{sec:negativity}) using a fixed $\lambda=0.8$ and the adaptive policy.}

\label{fig:cartpole_main}
\end{figure}
\begin{figure}[t]
    \centering

\includegraphics[scale=0.23]{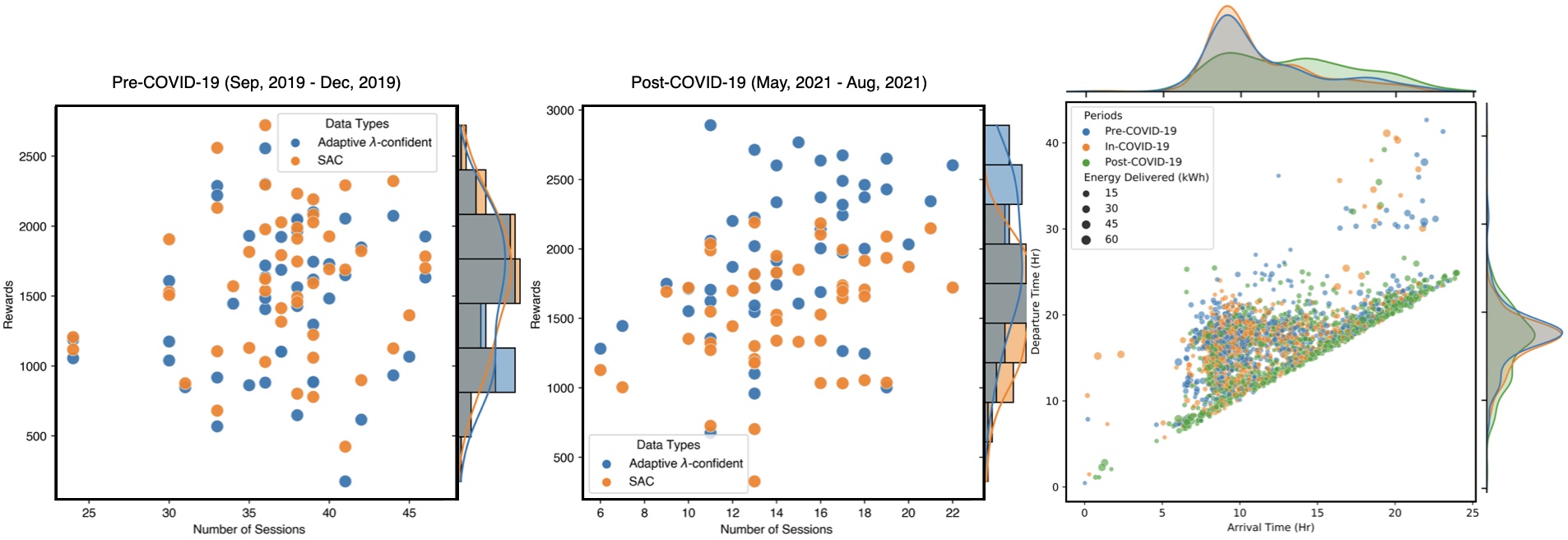}

\caption{Left: total rewards of the adaptive policy and SAC for pre-COVID-19 days and post-COVID-19 days. Right: shift of data distributions due to the work-from-home policy.}

\label{fig:ev_main}
\end{figure}

To demonstrate the efficacy of the adaptive $\lambda$-confident policy (Algorithm~\ref{alg:adaptive_policy}), we first apply it to the CartPole OpenAI gym environment (CartPole-v1, Example~\ref{example:cartpole} in Appendix~\ref{app:exp})~\cite{brockman2016openai}.\footnote{The CartPole environment is modified so that quadratic costs are considered rather than discrete rewards.}  Next, we apply it to an adaptive electric vehicle (EV) charging environment modeled by a real-world dataset~\cite{lee2019acn}.

\textbf{The CartPole problem.}
We use the Stable-Baselines3 pre-trained agents~\cite{stable-baselines3} of A2C~\cite{mnih2016asynchronous}, ARS~\cite{mania2018simple}, PPO~\cite{schulman2017proximal} and TRPO~\cite{schulman2015trust} as four candidate black-box policies.
In Figure~\ref{fig:cartpole_main}, the adaptive policy finds a trade-off between the pre-trained black-box polices and an LQR with crude model information (i.e., about $50\%$ estimation error in the mass and length values). In particular, when $\theta$ increases, it stabilizes the state while the A2C and PPO policies become unstable when the initial angle $\theta$ is large.

\textbf{Real-world adaptive EV charging.} In the EV charging application, a SAC~\cite{haarnoja2018soft} agent is trained with data collected from a pre-COVID-19 period and tested on days before and after COVID-19. Due to a policy change (the work-from-home policy), the SAC agent becomes biased in the post-COVID-19 period (see the right sub-figure in Figure~\ref{fig:ev_main}). With crude model information, the adaptive policy has rewards matching the SAC agent in the pre-COVID-19 period and significantly outperforms the SAC agent in the post-COVID-19 period with an average total award $1951.2$ versus $1540.3$ for SAC. Further details on the hyper-parameters and reward function are included in Appendix~\ref{app:exp}.

\section{Concluding Remarks}
\label{sec:con}

This work considers a novel combination of pre-trained black-box policies with model-based advice from crude model information. A general adaptive policy is proposed, with theoretical guarantees on both stability and sub-optimality. The effectiveness of the adaptive policy is validated empirically. We
believe that the results presented lead to an important first step towards improving the practicality of existing DNN-based algorithms when using them as black-boxes in nonlinear real-world control problems. Exploring  other forms of model-based advice theoretically, and verifying practically other implementations to learn the confidence coefficients online are interesting future directions.



\addcontentsline{toc}{section}{Bibliography}
\bibliographystyle{unsrt}
{\bibliography{main}}

\newpage

\appendix

\section{Experimental Setup and Supplementary Results}
\label{app:exp}

We describe the experimental settings and choices of hyper-parameters and reward/cost functions in the two applications.



\begin{table}[ht]
\footnotesize
\renewcommand{\arraystretch}{1.3}
    \centering
       \caption{Hyper-parameters used in the CartPole problem.}
    \begin{tabular}{l|l}
\specialrule{.11em}{.1em}{.1em} 
   \textbf{Parameter}  & \textbf{Value}  \\
   \hline
   Number of Monte Carlo Tests & $10$ \\
   Initial angle variation (in $\mathrm{rad}$) & $\theta\pm 0.05$ \\
   Cost matrix $Q$ & $I$ \\
   Cost matrix $R$ & $\left[10^{-4}\right]$\\
   Acceleration of gravity $g$ (in $\mathrm{m/s^2}$) & $9.8$\\
   Pole mass $m$ (in $\mathrm{kg}$) & $0.2$ for LQR; $0.1$ for real environment\\
   Cart mass $M$ (in $\mathrm{kg}$) & $2.0$ for LQR; $1.0$ for real environment  \\
   Pole length $l$ (in $\mathrm{m}$) &  $2$\\
   Duration $\tau$ (in $\mathrm{second}$) &  $0.02$\\
   Force magnitude $F$ & $10$ \\
    CPU  & Intel{\textregistered}  i7-8850H \\
\specialrule{.11em}{.1em}{.1em} 
    \end{tabular}
  \label{table:cartpole}
\end{table}

\subsection{The CartPole problem}

\textbf{Problem setting.}
The CartPole problem considered in the experiments is described by the following example.

\begin{example}[The CartPole Problem]
\label{example:cartpole}
In the CartPole problem, the goal of a controller is to stabilize the pole in the upright position. Neglecting friction, the dynamical equations of the CartPole problem are
\begin{align*}
    \ddot{\theta} = \frac{g \sin \theta + \cos \theta\left(\frac{-u-m l \dot{\theta}^2 \sin \theta}{m+M}\right)}{l\left(\frac{4}{3}-\frac{m \cos^2\theta}{m+M}\right)}, \ 
    \ddot{y} = \frac{u+ m l\left(\dot{\theta}^2 \sin\theta - \ddot{\theta}\cos \theta\right)}{m+M}
\end{align*}
where $u$ is the input force; $\theta$ is the angle between the pole and the vertical line; $y$ is the location of the pole; $g$ is the gravitational acceleration; $l$ is the pole length; $m$ is the pole mass; and $M$ is the cart mass. 
Taking $\sin\theta\approx \theta$ and $\cos\theta \approx 1$ and ignoring higher order terms provides a linearized system and the discretized dynamics of the CartPole problem can be represented as for any $t$,
\begin{align*}
\underbrace{\begin{bmatrix}
    y_{t+1}\\ \dot{y}_{t+1} \\ {\theta}_{t+1} \\ \dot{\theta}_{t+1}
    \end{bmatrix}}_{x_{t+1}} = \underbrace{\begin{bmatrix}
    1 & \tau & 0 & 0 \\ 0 & 1 & -\frac{m l g\tau}{\eta(m +M)} & 0 \\ 0 & 0 & 1 & \tau \\ 0 & 0 & \frac{g\tau}{\eta} & 1
    \end{bmatrix}}_{A}\underbrace{\begin{bmatrix}
    {y}_{t}\\ \dot{y}_{t} \\ {\theta}_{t}\\ \dot{\theta}_{t}
    \end{bmatrix}}_{x_t} + \underbrace{\begin{bmatrix}
  0\\ \frac{(m+M)\eta+m l}{(m+M)^2\eta}\tau\\ 0 \\ -\frac{\tau}{(m+M)\eta}
    \end{bmatrix}}_{B} u_t + f_t\left(y_{t},\dot{y}_{t},\theta_t,\dot{\theta}_t,u_t\right)
\end{align*}
where $(y_{t},\dot{y}_{t},\theta_t,\dot{\theta}_t)^\top$ denotes the system state at time $t$; $\tau$ denotes the time interval between state updates; $\eta\coloneqq ({4}/{3})l-{ml}/({m+M})$ and the function $f_t$ measures the difference between the  linearized system and the true system dynamics. Note that $f_t(\mathbf{0})=0$ for all time steps $t\geq 0$.
\end{example}

\textbf{Policy setting.}
The pre-trained agents  Stable-Baselines3~\cite{stable-baselines3} of A2C~\cite{mnih2016asynchronous}, ARS~\cite{mania2018simple}, PPO~\cite{schulman2017proximal} and TRPO~\cite{schulman2015trust} are selected as four candidate black-box policies. The CartPole environment is modified so that quadratic costs are considered rather than discrete rewards to match our control problem~\eqref{eq:system}. The choices of $Q$ and $R$ in the costs and other parameters are provided in Table~\ref{table:cartpole}. Note that we vary the values of $m$ and $M$ in the LQR implementation to model the case of only having crude estimates of linear dynamics. The LQR outputs an action $0$ if $-Kx_t+F'<0$ and $1$ otherwise. A shifted force $F'=15$ is used to model inaccurate linear approximations and noise. The pre-trained RL policies output a binary decision $\{0,1\}$ representing force directions. To use our adaptive policy in this setting, given a system state $x_t$ at each time $t$, we implement the following:
\begin{align*}
    u_t = \lambda_t \left(2\pi_{\mathrm{RL}}(x_t)F-F\right) + (1-\lambda_t)\left(2\pi_{\mathrm{LQR}}(x_t)F-F\right)
\end{align*}
where $F$ is a fixed force magnitude defined in Table~\ref{table:cartpole}; $\pi_{\mathrm{RL}}$ denotes an RL policy;  $\pi_{\mathrm{LQR}}$ denotes an LQR policy and $\lambda_t$ is a confidence coefficient generated based on~\eqref{eq:linear_lambda_t}. Instead of fixing a step size $\alpha$, we set an upper bound $\delta=0.2$ on the learned step size $\lambda'$ to avoid converging too fast to a pure model-based policy.

\begin{table}[t]
\footnotesize
\renewcommand{\arraystretch}{1.3}
    \centering
       \caption{Hyper-parameters used in the real-world EV charging problem.}
    \begin{tabular}{l|l}
\specialrule{.11em}{.1em}{.1em} 
   \textbf{Parameter}  & \textbf{Value}  \\
\hline
\multicolumn{2}{c}{\textit{Problem setting}}\\
\hline
   Number of chargers $\nn$ & $5$ \\
   Line limit $\gamma$ (in $\mathrm{kW}$) & $6.6$ \\
   Duration $\tau$ (in $\mathrm{minute}$)  & $5$ \\
   Reward coefficient $\phi_1$ & $50$ \\
   Reward coefficient $\phi_2$ & $0.01$\\
   Reward coefficient $\phi_3$ & $10$\\
     \hline
\multicolumn{2}{c}{\textit{Policy setting}}\\
\hline
 Discount  $\gamma_{\mathrm{SAC}}$  & $0.9$ \\
 Target smoothing coefficient $\tau_{\mathrm{SAC}}$  & $0.005$ \\
  Temperature parameter  $\alpha_{\mathrm{SAC}}$  & $0.2$ \\
Learning rate  & $3\cdot 10^{-4}$ \\
Maximum number of steps  & $10\cdot 10^{6}$ \\
Reply buffer size  & $10\cdot 10^{6}$ \\
Number of hidden layers (all networks) & $2$\\
Number of hidden units per layer & $256$\\
Number of samples per minibatch & $256$\\
Nonlinearity & $\mathrm{ReLU}$\\
\hline
\multicolumn{2}{c}{\textit{Training and testing data}}\\
\hline
   Pre-COVID-19 (Training)  &May, 2019 - Aug, 2019 \\
   Pre-COVID-19 (Testing)  &Sep, 2019 - Dec, 2019 \\
   In-COVID-19 &Feb, 2020 - May, 2020 \\
   Post-COVID-19 &May, 2021 - Aug, 2021 \\
    CPU  & Intel{\textregistered}  i7-8850H \\
\specialrule{.11em}{.1em}{.1em} 
    \end{tabular}
  \label{table:ev}
\end{table}

\subsection{Real-world adaptive EV charging in tackling COVID-19}

\textbf{Problem setting.}
The second application considered is an EV charging problem modeled by real-world large-scale charging data~\cite{lee2019acn}. The problem is formally described below.

\begin{example}[Adaptive EV charging]
\label{example:ev}
Consider the problem of managing a fleet of electric vehicle supply equipment (EVSE). Let $\nn$ be the number of EV charging stations. Denote by $x_t\in\mathbb{R}_+^\nn$ the charging states of the $\nn$ stations, i.e., $\smash{x_t^{(i)}>0}$ if an EV is charging at station-$i$ and $\smash{x_t^{(i)}}$ (kWh) energy needs to be delivered; otherwise  $\smash{x_t^{(i)}=0}$. Let $u_t\in\mathbb{R}_+^\nn$ be the allocation of energy to the $\nn$ stations. There is a line limit $\gamma>0$ so that $\smash{\sum_{i=1}^{\nn}u_t^{(i)}\leq \gamma}$ for any $i$ and $t$. At each time $t$, new EVs may arrive and  EVs being charged may depart from previously occupied stations. Each new EV $j$ induces a charging session, which can be represented by $s_j\coloneqq (a_j,d_j,e_j, i)$ where at time $a_j$, EV $j$ arrives at station $i$, with a battery capacity $e_j>0$, and depart at time $d_j$. Assuming lossless charging, the system dynamics is
$
    x_{t+1} = x_t + u_t + f_t(x_t,u_t), \ t\geq 0
$
where the nonlinear residual functions $(f_t:t\geq 0)$ represent uncertainty and constraint violations.
Let $\tau$ be the time interval between state updates. Given fixed session information $(s_j:j> 0)$, denote by the following sets containing the sessions that are assigned to a charger $i$ and activated (deactivated) at time $t$:
\begin{align*}
    \mathcal{A}_i &\coloneqq\left\{(j,t):a_j\leq t\leq a_j +\tau, s_j^{(4)}=i\right\},\\
    \mathcal{D}_i &\coloneqq\left\{(j,t):d_j\leq t\leq d_j +\tau, s_j^{(4)}=i\right\}.
\end{align*}
The charging uncertainty is summarized as for any $i=1,\ldots,\nn$,
\begin{align*}
f_t^{(i)}(x_t,u_t)\coloneqq \begin{cases}
s_j^{(3)}
\quad & \text{ if } (j,t)\in\mathcal{A}_i \hfill {\small\textsf{(New sessions are active)}}\\
-x^{(i)}_t-u^{(i)}_t
\quad & \text{ if } (j,t)\in\mathcal{D}_i \text{ or } x^{(i)}_t+u^{(i)}_t<0 \quad \hfill  {\small\textsf{(Sessions end/battery is full)}}\\
\frac{\gamma}{\|u_t\|_1}u^{(i)}_t
\quad & \text{ if } \sum_{i=1}^{\nn}u^{(i)}_t>\gamma \hfill {\small\textsf{(Line limit is exceeded)}}\\
0
\quad & \text{ otherwise }
 \end{cases}.
\end{align*}
\end{example}

\begin{figure}[t]
    \centering
\includegraphics[scale=0.2]{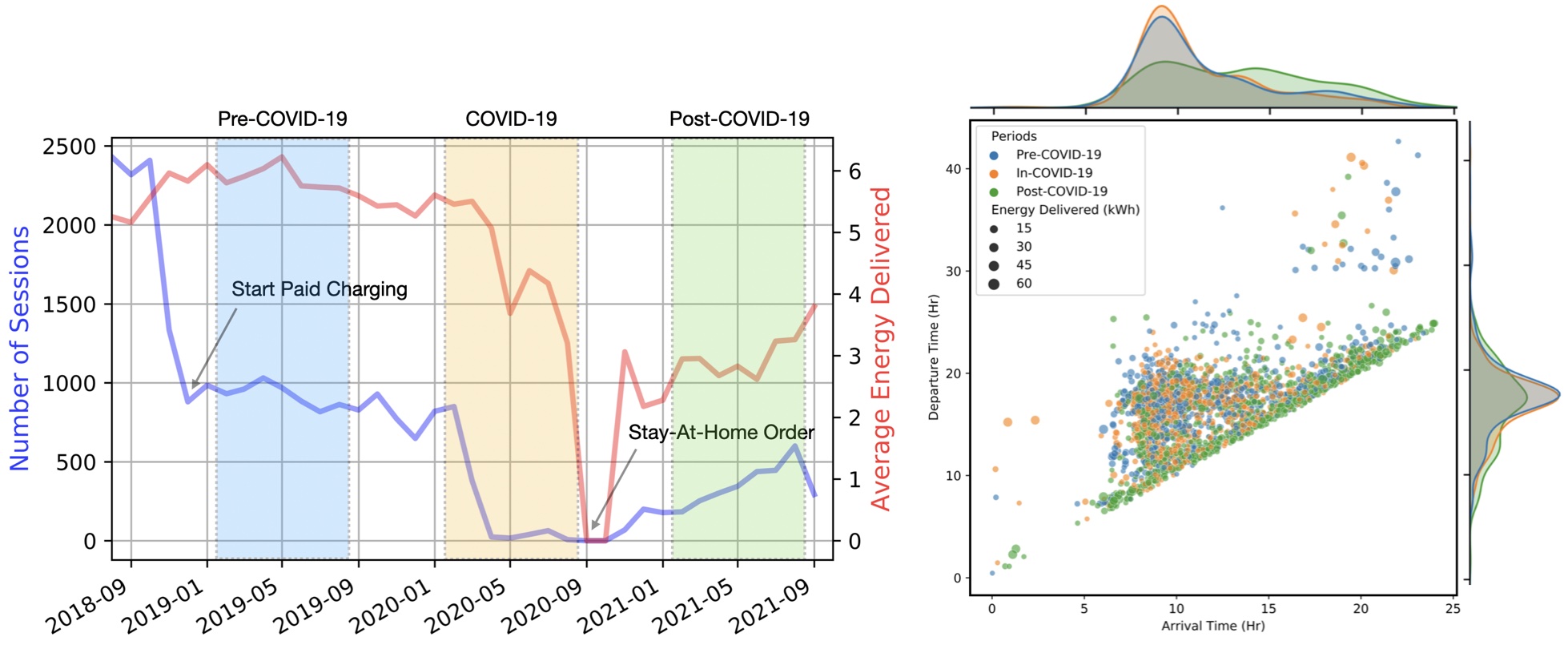}
\caption{Illustration of the impact of COVID-19 on charging behaviors in terms of the total number of charging sessions and energy delivered (left) and distribution shifts (right).}
\label{fig:data_shift}
\end{figure}

Note that the nonlinear residual functions $(f_t:t\geq 0)$ in Example~\ref{example:ev} may not satisfy $f_t(\mathbf{0})=0$ for all $t\geq 0$ in Assumption~\ref{assumption:continuity}. Our experiments further validate that the adaptive policy works well in practice even if some of the model assumptions are violated.
The goal of an EV charging controller is to maximize a system-level reward function including maximizing energy delivery, avoiding a penalty due to uncharged capacities and minimizing electricity costs. The reward function is
\begin{align*}
    r(u_t,x_t)\coloneqq \underbrace{\phi_1 \times \tau  \|u_t\|_2}_{\textrm{Charging rewards}}  -\underbrace{\phi_2 \times \|x_t\|_2}_{\textrm{Unfinished charging}} - \underbrace{\phi_3\times  p_t \|u_t\|_{1}}_{\textrm{Electricity cost}} - \underbrace{\phi_4\times \sum_{i=1}^{\nn} \mathbf{1}((j,t)\in \mathcal{D}_i)\frac{x_t^{(i)}}{e_j}}_{\textrm{Penalty}}
\end{align*}
with coefficients $\phi_1,\phi_2,\phi_3$ and $\phi_4$ shown in Table~\ref{table:ev}. The environment is wrapped as an OpenAI gym environment~\cite{brockman2016openai}. In our implementation, for convenience, the state $x_t$ is in $\mathbb{R}_{+}^{2\nn}$ with additional $\nn$ coordinates representing remaining charging duration. The electricity prices $(p_t:t\geq 0)$ are
average locational marginal prices (LMPs) on the CAISO (California Independent System Operator) day-ahead market in 2016.

\begin{figure}[t]
    \centering
\includegraphics[scale=0.215]{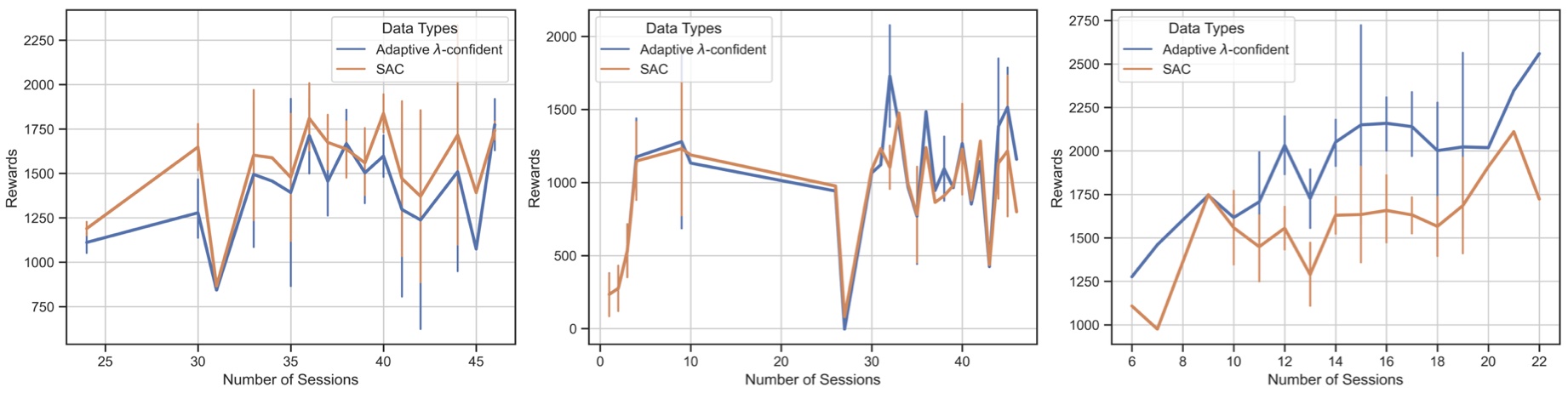}
\caption{Bar-plots of rewards/number of sessions corresponding to testing the SAC policy and the adaptive policy on the EV charging environment based on sessions collected from three time periods.}

\label{fig:reward_curve}
\end{figure}

\begin{figure}[t]
    \centering
\includegraphics[scale=0.205]{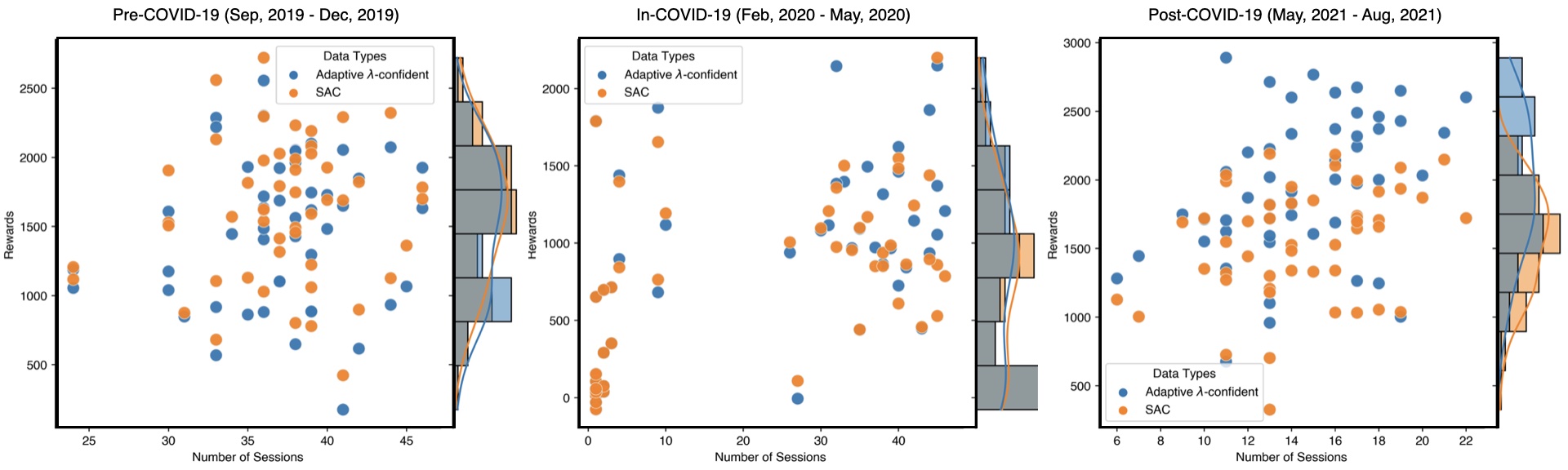}
\caption{Supplementary results of Figure~\ref{fig:ev_main} with additional testing rewards for an in-COVID-19 period.}

\label{fig:EV_full}
\end{figure}

\begin{table}[t]
\footnotesize
\renewcommand{\arraystretch}{1.3}
    \centering
       \caption{Average total rewards for the SAC policy and the adaptive $\lambda$-confident policy (Algorithm~\ref{alg:adaptive_policy}).}
    \begin{tabular}{l|c|c|c}
\specialrule{.11em}{.1em}{.1em} 
\textit{Policy} & \textbf{Pre-COVID-19}  & \textbf{In-COVID-19}  &
\textbf{Post-COVID-19} \\
   \hline
 SAC~\cite{haarnoja2018soft} & $1601.914$ & $765.664$  &  $1540.315$\\
 Adaptive & $1489.338$  & $839.651$ & $1951.192$\\
\specialrule{.11em}{.1em}{.1em} 
    \end{tabular}
  \label{table:rewards}
\end{table}

\textbf{Policy setting.}
We train an SAC~\cite{haarnoja2018soft} policy $\pi_{\mathrm{SAC}}$ for EV charging with $4$-month data collected from a real-world charging garage~\cite{lee2019acn} before the outbreak of COVID-19. The public charging infrastructure has $54$ chargers and we use the charging history to set up our charging environment with $5$ chargers. Knowledge of the linear parts in the nonlinear dynamics $
    x_{t+1} = x_t + u_t + f_t(x_t,u_t), \ t\geq 0
$ is assumed to be known, based on which an LQR controller $\pi_{\mathrm{LQR}}$ is constructed. Our adaptive policy presented in Algorithm~\ref{alg:adaptive_policy} learns a confidence coefficient $\lambda_t$ at each time step $t$ to combine the two policies $\pi_{\mathrm{SAC}}$ and $\pi_{\mathrm{LQR}}$.

\textbf{Impact of COVID-19.} We test the policies on different periods from $2019$ to $2021$. The impact of COVID-19 on the charging behavior is intuitive. As COVID-19 became an outbreak in early Feb, 2020 and later a pandemic in May, 2020, limited Stay at Home Order and curfew were issued, which significantly reduce the number of active users per day. Figure~\ref{fig:data_shift} illustrates the dramatic fall of the total number of monthly charging sessions and total monthly energy delivered between Feb, 2020 and Sep, 2020. Moreover, despite the recovery of the two factors since Jan, 2021, COVID-19 has a long-term impact on lifestyle behaviors. For example, the right sub-figure in Figure~\ref{fig:data_shift} shows that the arrival times of EVs (start times of sessions) are flattened in the post-COVID-19 period, compared to a more concentrated arrival peak before COVID-19. The significant shift of  distributions highly deteriorates the performance of DNN-based model-free policies, e.g., SAC that are trained on normal charging data collected before COVID-19. In this work, we demonstrate that taking advantage of model-based information, the adaptive $\lambda$-confident policy (Algorithm~\ref{alg:adaptive_policy}) is able to correct the mistakes made by DNN-based model-free policies trained on biased data and achieve more robust charging performance.

\textbf{Additional experimental results.}
We provide supplementary experimental results. Besides comparing the periods of pre-COVID-19 and post-COVID-19, we include the testing rewards for an in-COVID-19 period in Figure~\ref{fig:EV_full}, together with the corresponding bar-plots in Figure~\ref{fig:reward_curve}. In addition, the average total rewards for the SAC policy and the adaptive policy are summarized in Table~\ref{table:rewards}.

\section{Notation and Supplementary Definitions}
\label{app:notation}

\subsection{Summary of notation}

A summary of notation is provided in Table~\ref{table:parameter}.

\begin{table}[ht]
\footnotesize
\renewcommand{\arraystretch}{1.3}
    \centering
    \begin{tabular}{c|c}
\specialrule{.15em}{.1em}{.1em} 
  \textbf{Symbol}  & \textbf{Definition}  \\
  \hline
  \multicolumn{2}{c}{\textit{System Model}}\\
  \hline
 $A,B,Q,R$ & Linear system parameters \\
  $P$ & Solution of the DARE~\eqref{eq:dare} \\
  $H$ & $R+B^{\top}PB$ \\
  $K$ & $H^{-1} B^\top P A$\\
  $F$ &  $A-BK$  \\
$\sigma$ &  $Q,R\succeq \sigma I$\\
  $\kappa$ & $\max\{2,\|A\|,\|B\|\}$\\
   $C_F$ and $\rho$ &  $\| F^t\|\leq C_F\rho^t$\\
   $f_t:\mathbb{R}^{\nn}\times\mathbb{R}^{\nm}\rightarrow\mathbb{R}^{\nn}$  & Nonlinear residual functions\\
   $C_{\ell}$ & Lipschitz constant of $f_t$\\
  $\widehat{\pi}$ & Black-box (model-free) policy\\
   $\overline{\pi}$ & Model-based policy (LQR)\\
$\varepsilon$ & Consistency error of a black-box policy\\
   \hline
\multicolumn{2}{c}{\textit{Main Results}}\\
\hline
$C_{a}^{\mathsf{sys}},C_{b}^{\mathsf{sys}},C_{c}^{\mathsf{sys}}$ & Constants defined in~Section~\ref{sec:theorem_constants} \\
$\gamma$ & $\rho+C_F C_{\ell}(1+\|K\|)$\\
$\mu$ & $C_F\left(\varepsilon\left(C_{\ell}+\|B\|\right)+C_{a}^{\mathsf{sys}} C_{\ell}\right)$\\
$\mathsf{ALG}$ & Algorithm cost\\
$\mathsf{OPT}$ & Optimal cost\\
$\mathsf{CR}(\varepsilon)$ & $2\kappa(\frac{{{C_F}}\|P\|}{ 1-\rho})^2/\sigma$\\
$\lambda$ & $\lim_{t\rightarrow\infty}\lambda_t$\\
$t_0$ & The smallest time index when $\lambda_t = 0$ or $x_t=\mathbf{0}$\\
\specialrule{.15em}{.1em}{.1em} 
    \end{tabular}
  \vspace{3pt}
\caption{Symbols used in this work.}
  \label{table:parameter}
\end{table}
\subsection{Constants in Theorem~\ref{thm:stability} and~\ref{thm:competitive}}
\label{sec:theorem_constants}
Let $H\coloneqq  R+B^{\top}PB$. With $\sigma>0$ defined in Assumption~\ref{assumption:stability}, the parameters  $C_{a}^{\mathsf{sys}}, C_{b}^{\mathsf{sys}},C_{c}^{\mathsf{sys}}>0$ in the statements of Theorem~\ref{thm:stability} and~\ref{thm:competitive} (Section~\ref{sec:main}) are the following constants that only depend on the known system parameters in~\eqref{eq:system}:

\begin{align}
\nonumber
C_{a}^{\mathsf{sys}}\coloneqq & 1/\Big(2C_F\|R+B^\top PB\|^{-1}\Big(\|PF\|+(1+\|K\|)\left(\|PB\|+\|P\|\right)\\
  \label{eq:policy_constant}
 & \quad +\frac{C_{b}^{\mathsf{sys}}}{2}\|B+ I\|(1+\|F\|+\|K\|)\Big)\Big),\\
  \nonumber
C_{b}^{\mathsf{sys}}\coloneqq & \frac{2{{C_F}^2}\|P\|(\rho+\overline{C})\left(\rho+(1+\|K\|)\right)}{1-(\rho+\overline{C})^2}\sqrt{\frac{\|Q+K^\top RK\|}{\sigma}},\\
\nonumber
C_{c}^{\mathsf{sys}}\coloneqq &{\|H\|}/{\left(4\left\| PB\right\|+2\|P\|+C_\nabla(\|B\|+1)\|B\|\right)}.
\end{align}

\subsection{Approximations in online learning steps~Section~\ref{sec:online_learning}}

The following approximations of $ \eta(\widehat{f};s,t)$ and $  \eta(f^*;s,t)$ in~\eqref{eq:optimal_lambda} are used to derive the expression of $\lambda'$ in~\eqref{eq:linear_lambda_t} for learning the confidence coefficients online:
\begin{align}
\label{eq:linear_policy}
  \eta(\widehat{f};s,t) &\approx  \sum_{\tau=s}^{\infty}\left(F^\top\right)^{\tau-s} P \widehat{f}_\tau = -\left(H^{-1}B^{\top}\right)^\dagger\left(\widehat{u}_s + Kx_s \right),\\
  \label{eq:linear_dynamic}
  \eta(f^*;s,t) &=\sum_{\tau=s}^{t}\left(F^\top\right)^{\tau-s} P f^*_\tau \approx  \sum_{\tau=s}^{t}\left(F^\top\right)^{\tau-s} P (x_{\tau+1}-Ax_{\tau}-Bu_{\tau}).
\end{align}

\section{Useful Lemmas}

The following lemma generalizes the results in~\cite{yu2020competitive}.

\begin{lemma}[Generalized cost characterization lemma]
\label{lemma:characterization}
Consider a linear quadratic control problem below where $Q,R\succ 0$ and the pair of matrices $(A,B)$ is stabilizable:
\begin{align}
\nonumber
\min_{(u_t: t\geq 0)} \sum_{t=0}^{\infty}& (x_t^\top Q x_t + u_t^\top R u_t), \
\textrm{subject} \ \textrm{to} \ x_{t+1} = Ax_t+Bu_t+v_t \textrm{ for any } t\geq 0.
\end{align}
If at each time $t\geq 0$, $u_t =  -Kx_t-H^{-1}B^{\top}W_t+\eta_t$ where $\eta_t\in\mathbb{R}^\nm$,
then the induced cost is
\begin{align*}
& x_0^{\top}Px_0 + 2x_0^{\top}F^{\top}V_0 + \sum_{t=0}^{\infty}\eta_t^{\top} H \eta_t+ \sum_{t=0}^{\infty}\left(v_t^{\top}P v_t + 2v_t^{\top}F^{\top}V_{t+1}\right)\\
    & + \sum_{t=0}^{\infty}\left(W_t^{\top}B H^{-1} B^{\top}(W_t-2V_t) + 2\eta_t^{\top} B^{\top}(V_t-W_t) \right) + O(1)
\end{align*}
where $P$ is the unique solution of the  DARE in~\eqref{eq:dare}, $H \coloneqq R+B^{\top}PB$, $F\coloneqq A-B(R+B^\top P B)^{-1} B^\top PA = A-BK$, $W_t \coloneqq {\sum_{\tau=t}^{\infty}}\left(F^{\top}\right)^{\tau}Pw_{t+\tau}$  and
    $V_t \coloneqq {\sum_{\tau=t}^{\infty}}\left(F^{\top}\right)^{\tau}Pv_{t+\tau}$.
\end{lemma}

\begin{proof}
Denote by $\mathsf{COST}_t(x_t;\eta,w)$ the terminal cost at time $t$ given a state $x_t$ with fixed action perturbations $\eta\coloneqq (\eta_t:t\geq 0)$ and state perturbations $w\coloneqq (w_t: t\geq 0)$. 
We assume $\mathsf{COST}_t(x_t;\eta,w)
\coloneqq x_t^{\top}P x_t + p_t^{\top} x_t + q_t$. 
Similar to the proof of Lemma 13 in~\cite{yu2020competitive}, using the backward induction, the cost can be rewritten as
\begin{align*}
\mathsf{COST}_t(x_t;\eta,w) =& \mathsf{COST}_{t+1}(x_{t+1};\eta,w) + x_t^{\top} Q x_t + u_t^{\top} R u_t \\
=&x_t^{\top} Q x_t + u_t^{\top} R u_t + (Ax_t+Bu_t+v_t)^{\top} P (Ax_t+Bu_t+v_t) \\
& \quad + p_{t+1}^{\top}(Ax_t+Bu_t+v_t) + q_{t+1}\\
=& x_t^{\top}Qx_t + (Ax_t+v_t)^{\top}P(Ax_t+v_t) + (Ax_t+v_t)^{\top}p_{t+1}+q_{t+1}\\
&\quad + \underbrace{u_t^{\top}(R+B^{\top}PB)u_t}_{(a)} + \underbrace{2u_t^{\top}B^{\top}(PAx_t + Pv_t+p_{t+1}/2)}_{(b)}.
\end{align*}
Denote by $H \coloneqq R+B^{\top}PB$. Noting that $u_t=-Kx_t+G_t+\eta_t$ where we denote 
\begin{align*}
    G_t \coloneqq H^{-1}B^{\top}W_t, \  W_t \coloneqq {\sum_{\tau=t}^{\infty}}\left(F^{\top}\right)^{\tau}Pw_{t+\tau} \ \text{and} \
    V_t \coloneqq {\sum_{\tau=t}^{\infty}}\left(F^{\top}\right)^{\tau}Pv_{t+\tau},
\end{align*}
it follows that
\begin{align*}
    (a)=&\left(Kx_t+G_t-\eta_t\right)^{\top}H\left(Kx_t+G_t-\eta_t\right)-2(Kx_t)+2(Kx_t)^{\top}H(G_t-\eta_t)\\
&\quad + (G_t-\eta_t)^{\top}(R+B^{\top}PB)(G_t-\eta_t)\\
(b)=&-2(Kx_t)^{\top}H(Kx_t) -2x_t^{\top}K^{\top} B^{\top} P v_t -x_t^{\top}K^{\top}B^{\top}p_{t+1}-2(Kx_t)^{\top}H(G_t-\eta_t)\\
&\quad -2(G_t-\eta_t)^{\top}B^{\top}(Pv_t+p_{t+1}/2),
\end{align*}
implying  
\begin{align*}
\mathsf{COST}_t(x_t;\eta,w)=&x_t^{\top} (Q+A^{\top}PA-K^{\top}HK)x_t + x_t^{\top}F^{\top}(2Pv_t +p_{t+1})  \\
&+ (G_t-\eta_t)^{\top} H (G_t-\eta_t) - 2(G_t-\eta_t)^{\top}B^{\top}(Pv_t+p_{t+1}/2)\\
&+ v_t^{\top}P v_t + v_t^{\top}p_{t+1} + q_{t+1}.
\end{align*}
According to the DARE in~\eqref{eq:dare}, since $K^{\top}HK= A^\top PB (R+B^\top P B)^{-1} B^\top PA$, we get
\begin{align*}
\mathsf{COST}_t(x_t;\eta,w)&=x_t^{\top}Px_t +  x_t^{\top}F^{\top}(2Pv_t +p_{t+1})  + (G_t-\eta_t)^{\top} H (G_t-\eta_t)  \\
&\quad -  2(G_t-\eta_t)^{\top}B^{\top}(Pv_t+p_{t+1}/2) + v_t^{\top}P v_t + v_t^{\top}p_{t+1} + q_{t+1},
\end{align*}
which implies
\begin{align}
\label{eq:v_t}
    p_t =& 2F^{\top}\left(Pv_t + p_{t+1}\right) =2{\sum_{\tau=t}^{\infty}} \left(F^{\top}\right)^{\tau+1}P v_{t+\tau} = 2F^{\top}V_t,\\
    \nonumber
    q_t =&  q_{t+1} + v_t^{\top}P v_t + 2v_t^{\top}F^{\top}V_{t+1} +  (G_t-\eta_t)^{\top} H (G_t-\eta_t)  \\
    \nonumber
&\quad - 2(G_t-\eta_t)^{\top}B^{\top}(Pv_t+p_{t+1}/2)\\
\nonumber
=&  q_{t+1} + v_t^{\top}P v_t + 2v_t^{\top}F^{\top}V_{t+1} +  (G_t-\eta_t)^{\top} H (G_t-\eta_t)  \\
\nonumber
&\quad - 2G_t^{\top}B^{\top}V_t + 2\eta_t^{\top} B^{\top}V_t\\
\label{eq:q_t}
=&  q_{t+1} + v_t^{\top}P v_t + 2v_t^{\top}F^{\top}V_{t+1} + G_t^{\top} B^{\top}(W_t-2V_t) + 2\eta_t^{\top} B^{\top}(V_t-W_t) + \eta_t^{\top} H \eta_t.
\end{align}
Therefore,~\eqref{eq:v_t} and~\eqref{eq:q_t} together imply the following general cost characterization:
\begin{align*}
    \mathsf{COST}_t(x_t;\eta,w) =& x_0^{\top}Px_0 + x_0^{\top}p_0 + q_0\\
    = & x_0^{\top}Px_0 + 2x_0^{\top}F^{\top}V_0 + \sum_{t=0}^{\infty}\eta_t^{\top} H \eta_t+ \sum_{t=0}^{\infty}\left(v_t^{\top}P v_t + 2v_t^{\top}F^{\top}V_{t+1}\right)\\
    & + \sum_{t=0}^{\infty}\left(G_t^{\top} B^{\top}(W_t-2V_t) + 2\eta_t^{\top} B^{\top}(V_t-W_t) \right).
\end{align*} 
Rearranging the terms above completes the proof.
\end{proof}


To deal with the nonlinear dynamics in~\eqref{eq:nonlinear_dynamic}, we consider an auxiliary linear system, with a fixed perturbation $w_t=f_t(x_t^*,u_t^*)$ for all $t\geq 0$ where each $x_t^*$ denotes an optimal  state and $u_t^*$ an optimal action, generated by an optimal policy $\pi^*$. We define a sequence of linear policies $\left(\pi'_t: t\geq 0\right)$ where $\pi_t':\mathbb{R}^{\nn}\rightarrow\mathbb{R}^{\nm}$ generates an action
\begin{align}
\label{eq:auxiliary}
u_t'=\pi_t'(x_t)\coloneqq-(R+B^{\top}PB)^{-1}B^{\top}\left(PAx_t+\sum_{\tau=t}^{\infty}\left(F^\top\right)^{\tau-t} P f_{\tau}(x_\tau^*,u_\tau^*) \right),
\end{align}
which is an optimal policy for the auxiliary linear system.
Utilizing Lemma~\ref{lemma:characterization}, the gap between the optimal cost  and algorithm cost for the system in~\eqref{eq:system} can be characterized below.

\begin{lemma}
\label{lemma:gap}
For any $\eta_t\in\mathbb{R}^\nm$, if at each time $t\geq 0$, a policy $\pi:\mathbb{R}^{\nn}\rightarrow\mathbb{R}^{\nm}$ takes an action
$u_t = \pi(x_t) = \pi_t'(x_t) + \eta_t$,
then the gap between the optimal cost $\mathsf{OPT}$ of the nonlinear system~\eqref{eq:system} and the algorithm cost $\mathsf{ALG}$ induced by selecting control actions $(u_t: t\geq 0)$ equals to
\begin{align}
\nonumber
\mathsf{ALG} -\mathsf{OPT}  
\leq &\sum_{t=0}^{\infty}\eta_t^{\top} H \eta_t  + O(1)\\
\nonumber
& + 2\sum_{t=0}^{\infty}\eta_t^{\top}B^{\top}\left({\sum_{\tau=t}^{\infty}}\left(F^{\top}\right)^{\tau}P(f_{t+\tau}-f^*_{t+\tau})\right)\\
\nonumber
& + \sum_{t=0}^{\infty}\left(f_t^{\top}Pf_t -\left(f_t^*\right)^{\top}Pf_t^* \right) 
+ 2x_0^{\top}\left(\sum_{t=0}^{\infty}\left(F^{\top}\right)^{t+1}P\left(f_{t}-f^*_{t}\right)\right)
\\
\nonumber
& + 2\sum_{t=0}^{\infty}\left(f_t^{\top}{\sum_{\tau=0}^{\infty}}\left(F^{\top}\right)^{\tau+1}Pf_{t+\tau+1} -\left(f_t^*\right)^{\top}{\sum_{\tau=0}^{\infty}}\left(F^{\top}\right)^{\tau+1}P\left(f_{t+\tau+1}^*\right)\right) \\
\label{eq:gap}
& + 2\sum_{t=0}^{\infty}\left({\sum_{\tau=t}^{\infty}}\left(F^{\top}\right)^{\tau}Pf^*_{t+\tau}\right)BH^{-1}B^{\top}\left({\sum_{\tau=t}^{\infty}}\left(F^{\top}\right)^{\tau}P(f^*_{t+\tau}-f_{t+\tau})\right)
\end{align}
where $H\coloneqq R+B^\top P B$ and $F\coloneqq A-BK$. For any $t\geq 0$, we write $f_t\coloneqq f_t(x_t,u_t)$ and $f_t^*\coloneqq f_t(x^*_t,u^*_t)$ where $(x_t: t\geq 0)$ denotes a trajectory of states generated by the policy $\pi$ with  actions $(u_t^*: t\geq 0)$ and  $(x_t^*: t\geq 0)$ denotes an optimal trajectory of states generated by optimal actions $(u_t^*: t\geq 0)$.
\end{lemma}

\begin{proof}
Note that with optimal trajectories of states and actions fixed, the optimal controller $\pi^*$ induces the same cost $\mathsf{OPT}$ for both the nonlinear system in~\eqref{eq:system} and the auxiliary linear system. Moreover, the linear controller defined in~\eqref{eq:auxiliary} induces a cost ${\mathsf{OPT}'}$ that is smaller than $\mathsf{OPT}$ when running both in the auxiliary linear system since the constructed linear policy is optimal. Therefore, according to Lemma~\ref{lemma:characterization},
$
\mathsf{ALG}-\mathsf{OPT}\leq \mathsf{ALG}-\mathsf{OPT}'
$
and
applying Lemma~\ref{lemma:characterization} with $v_t=f_t(x_t,u_t)$ and $w_t=f_t(x_t^*,u_t^*)$ for all $t\geq 0$,~\eqref{eq:gap} is obtained.
%
\end{proof}

\section{Proof of Theorem~\ref{thm:negativity}}
\label{app:proof_negativity}
\begin{proof}
Fix an arbitrary controller $K_1$ with a closed-loop system matrix $F_1\coloneqq A-BK_1\neq 0$. We first consider the case when $F_1$ is not a diagonal matrix, i.e.,  $A-BK_1$ has at least one non-zero off-diagonal entry. Consider the following closed-loop system matrix for the second controller $K_2$:
\begin{align}
\label{eq:construction_of_F}
F_2\coloneqq
    \left(
    \begin{array}{ccccc}
    \beta  &  & \cdots  \\
      & \beta       &   & \text{\Large  0} \\
      &               & \ddots                \\
      & -\frac{1-\lambda}{\lambda}\overline{L} &   & \beta           \\
      &               &   &   & \beta
    \end{array}
    \right) + S\left(F_1\right)
\end{align} 
where $0<\beta<1$ is a value of the diagonal entry; $\overline{L}$ is the lower triangular part of the closed-loop system matrix  $F_1$. The matrix $S(F_1)$ is a singleton matrix that depends on $F_1$, whose only non-zero entry $S_{i+k,i}$ corresponding to the first non-zero off-diagonal entry $(i,i+k)$ in the upper triangular part of $F_1$ searched according to the order $i=1,\ldots,\nn$ and $k=1,\ldots,\nn$ with $i$ increases first and then $k$. Such a non-zero entry always exists because in this case $F_1$ is not a diagonal matrix. If the non-zero off-diagonal entry appears to be in the lower triangular part, we can simply transpose $F_2$ so without loss of generality we assume it is in the upper triangular part of $F_1$.
Since $F_1$ is a lower triangular matrix, all of its eigenvalues equal to $0<\beta<1$, implying that the linear controller $K_2$ is stabilizing.
Then, based on the construction of $F_1$ in~\eqref{eq:construction_of_F}, the linearly combined controller $K\coloneqq\lambda K_2 + (1-\lambda) K_1$ has a closed-loop system matrix $F$ which is upper-triangular, whose determinant satisfies
\begin{align}
\nonumber
    \det(F)=&\det\left(\lambda F_2 + (1-\lambda) F_1\right) \\
    \label{eq:det_1}
    = & (-1)^{2i+k} \det(F')\\
    \label{eq:det_2}
    =& (-1)^{2i+k}\times (-1)^{k+1}(1-\lambda)\lambda^{\nn-1} S_{i+k,i}(F_1)_{i,i+k}\beta^{\nn-2}\\
    \label{eq:det_3}
    =&- S_{i+k,i}(1-\lambda)\lambda^{\nn-1}\beta^{\nn-2}
\end{align}
where the term $(-1)^{2i+k}$ in~\eqref{eq:det_1} comes from the computation of the determinant of $F$ and $F'$ is a sub-matrix of $F$ by eliminating the $i+k$-th row and $i$-th column. The term $(-1)^{k+1}$ in~\eqref{eq:det_2}  appears because $F'$ is a permutation of an upper triangular matrix with $\nn-2$ diagonal entries being $\beta$'s and one remaining entry being $(F_1)_{i,i+k}$  since the entries $S_{j,j+k}$ are zeros all $j<i$; otherwise another non-zero entry $S_{i+k,i}$ would be chosen according to our search order.

Continuing from~\eqref{eq:det_3}, since $S_{i+k,i}$ can be selected arbitrarily, setting $S_{i+k,i}=\frac{{-2^{\nn}}{\beta(\lambda\beta)^{-\nn+1}}}{1-\lambda}$ gives
\begin{align*}
  2^{\nn} \leq \det(F) \leq |\rho(F)|^{\nn},
\end{align*}
implying that the spectral radius $\rho(F)\geq 2$. Therefore $K=\lambda K_2 + (1-\lambda) K_1$ is an unstable controller. It remains to prove the theorem in the case when $F_1$ is a diagonal matrix. In the following lemma, we  consider the case when $\nn=2$, and extend  it to the general case.

\begin{lemma}
\label{lemma:2dcase}
For any $\lambda\in(0,1)$ and any diagonal matrix $F_1\in \mathbb{R}^{2\times2}$ with a spectral radius $\rho(F_1)<1$ $F_1\neq \gamma I$ for any $\gamma\in \mathbb{R}$, there exists a matrix $F_2\in \mathbb{R}^{2\times2}$ such that $\rho(F_2)<1$ and $\rho(\lambda F_1+(1-\lambda)F_2)>1$.
\end{lemma}
\begin{proof}[Proof of Lemma~\ref{lemma:2dcase}]
Suppose $\nn=2$ and 
\begin{align*}
    F_1=\begin{bmatrix}
    a& 0\\
    0& b
    \end{bmatrix}
\end{align*}
is a diagonal matrix and without loss of generality assume that $a>0$ and $a>b$. Let 
\begin{align}
\label{eq:2by2F}
    F_2=\frac{4}{\lambda(1-\lambda)(a-b)}\begin{bmatrix}
    1 & 1\\
    -1 & -1
    \end{bmatrix}.
\end{align}
Notice that $\rho(F_2)=0<1$ satisfies the constraint. Then
\begin{align}
\nonumber
   \rho\left(\lambda F_1+(1-\lambda)F_2\right)=&\rho(\lambda b I +\lambda \mathrm{Diag}(a-b,0)+(1-\lambda)F_2)\\
   \label{eq:equality_rho}
    =&\lambda b + \rho(\lambda \mathrm{Diag}(a-b,0)+(1-\lambda)F_2)
\end{align}
where we have used the assumption that $a>0$ and $a>b$ to derive the equality in~\eqref{eq:equality_rho} and the notion $\mathrm{Diag}(a-b,0)$ denotes a diagonal matrix whose diagonal entries are $a-b$ and $0$ respectively.
The eigenvalues of $\lambda \mathrm{Diag}(a-b,0)+(1-\lambda)F_2$ are 
\begin{align*}
    \frac{\lambda(a-b)\pm\sqrt{(\lambda(a-b)+\frac{8}{\lambda(a-b)})^2-4(\frac{4}{\lambda(a-b)})^2}}{2}=\frac{\lambda(a-b)\pm\sqrt{(\lambda(a-b))^2+16}}{2}.
\end{align*}
Since $a-b>0$, the spectral radius of $\lambda F_1+(1-\lambda)F_2$ satisfies
\begin{align*}
     \rho\left(\lambda F_1+(1-\lambda)F_2\right)=\lambda b + \frac{\lambda(a-b)+\sqrt{(\lambda(a-b))^2+16}}{2}>-1+\frac{\sqrt{16}}{2}=1.
\end{align*}
\end{proof}

Applying Lemma~\ref{lemma:2dcase}, when $F_1$ is an $\nn\times\nn$ matrix with $\nn>2$, we can always create an $\nn\times\nn$ matrix $F_2$ whose first two columns and rows form a sub-matrix that is the same as~\eqref{eq:2by2F} and the remaining entries are zeros. Therefore the spectral radius of the convex combination $\lambda F_1 + (1-\lambda)F_2$ is greater than one. This completes the proof.
\end{proof}

\section{Stability Analysis}
\label{app:proof_stability}

\subsection{Proof outline}

\begin{figure}[h]
    \centering
\includegraphics[scale=0.23]{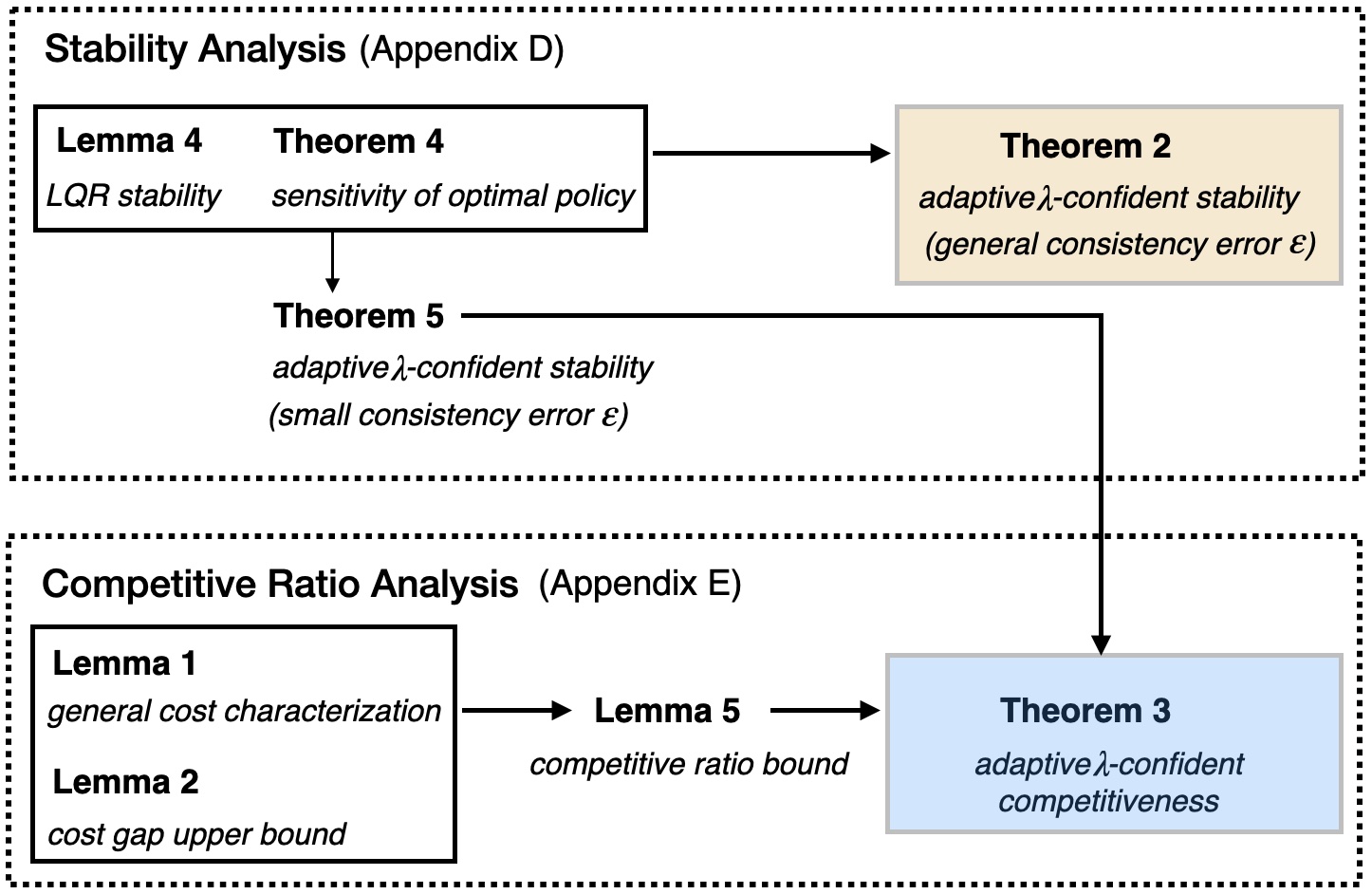}
\caption{Outline of proofs of Theorem~\ref{thm:stability} and~\ref{thm:competitive} with a stability analysis presented in Appendix~\ref{app:proof_stability} and a competitive ratio analysis presented in Appendix~\ref{app:proof_competitiveness}. Arrows denote implications.}
\label{fig:proof_outline}
\end{figure}

In the sequel, we present the proofs of Theorem~\ref{thm:stability} and~\ref{thm:competitive}. An outline of the proof structure is provided in Figure~\ref{fig:proof_outline}.
The proof of our main results contains two parts -- the \textit{stability analysis}  and \textit{competitive ratio analysis}. First, we prove that Algorithm~\ref{alg:adaptive_policy} guarantees a stabilizing policy, regardless of the prediction error $\varepsilon$ (Theorem~\ref{thm:stability}). Second, in our competitive ratio analysis, we provide a competitive ratio bound in terms of $\varepsilon$ and $\lambda$. We show in Lemma~\ref{lemma:competitive} that the competitive ratio bound is bounded if the adaptive policy is exponentially stabilizing and has a decay ratio that scales up with $C_{\ell}$, which holds assuming the prediction error $\varepsilon$ for a black-box model-free policy is small enough, as shown in Theorem~\ref{thm:blackbox_stability}. Theorem~\ref{thm:blackbox_stability} is proven based on a sensitivity analysis of an optimal policy $\pi^*$ in Theorem~\ref{lemma:model_based_exponential_stability}.

We first analyze the model-based policy $\widehat{\pi}(x)=-Kx$ where $K \coloneqq (R+B^\top P B)^{-1} B^\top P A$ and $P$ is a unique solution of the Riccati equation in~\eqref{eq:dare}.
\begin{lemma}
\label{lemma:model_based_exponential_stability}
Suppose the Lipschitz constant $C_{\ell}$, $K$ and the closed-loop matrix $F\coloneqq A-BK$ satisfy $\rho+C_F C_{\ell}(1+\|K\|)<1$ where $\left\| F^t\right\|\leq C_F\rho^t$ for any $t\geq 0$. Then the model-based policy $\widehat{\pi}(x)=-Kx$  exponentially stabilizes the system such that $    \|x_t\|\leq C_F  \left(\rho+{C_F\overline{C}}\right)^{t}\|x_0\|$ for any $t\geq 0$.
\end{lemma}

\begin{proof}
Let $u_t=\widehat{\pi}(x_t)=-Kx_t$ for all $t\geq 0$ and  let $F\coloneqq A-BK$. It follows that
\begin{align}
\label{eq:stability_proof_dynamic}
    x_{t+1}=Ax_t+Bu_t+f_t(x_t,u_t)=(A-BK)x_t+f_t(x_t,-Kx_t)=Fx_t+f_t(x_t,-Kx_t).
\end{align}
Rewriting~\eqref{eq:stability_proof_dynamic} recursively, for any $t\geq 0$,
\begin{align}
\label{eq:stability_proof_1}
     x_t=F^{t}x_0+\sum_{\tau=0}^{t-1}F^{t-1-\tau} f_{\tau}(x_\tau,-Kx_\tau).
\end{align}

Since $(f_t:t\geq 0)$ are Lipschitz continuous with a constant $C_{\ell}$ (Assumption~\ref{assumption:continuity}), we have
\begin{align}
\nonumber
    \|x_t\|&\leq\left\|F^{t}x_0\right\|+\left\|\sum_{\tau=0}^{t-1}F^{t-1-\tau} f_{\tau}(x_\tau,-Kx_\tau)\right\|\\
    \nonumber
    &\leq\left\|F^{t}x_0\right\|+\sum_{\tau=0}^{t-1}\left\|F^{t-1-\tau} f_{\tau}(x_\tau,-Kx_\tau)\right\|\\
    \label{eq:stability_proof_2}
    &\leq C_F\rho^t\underbrace{\left(\left\|x_0\right\|+C_{\ell} (1+\|K\|)\sum_{\tau=0}^{t-1}\rho^{-1-\tau}\|x_{\tau}\|\right)}_{:=S_t}
\end{align}
where $C_F>1$ is a constant such that $\left\| F^t\right\|\leq C_F\rho^t$ for any $t\geq 0$. Denote by $\overline{C}\coloneqq C_{\ell} (1+\|K\|)$. Then, using~\eqref{eq:stability_proof_2},
\begin{align*}
S_t&=  S_{t-1}+\overline{C}\rho^{-t} \|x_{t-1}\|\leq  S_{t-1}+\frac{C_F\overline{C}}{\rho} S_{t-1} = \left(1+\frac{C_F\overline{C}}{\rho}\right)S_{t-1}.
\end{align*} 
Therefore, noting that  $S_1=\left(1+{\overline{C}}/{\rho}\right)\|x_0\|$, recursively we obtain $$S_t\leq \left(1+\frac{C_F\overline{C}}{\rho}\right)^{t-1}\left(1+\frac{\overline{C}}{\rho}\right)\|x_0\|,$$ which implies
\begin{align*}
\|x_t\|\leq C_F \rho^t S_t \leq & C_F \rho^t \left(1+\frac{C_F\overline{C}}{\rho}\right)^{t-1}\left(1+\frac{\overline{C}}{\rho}\right)\|x_0\|\\
= &C_F  \left(\rho+{C_F\overline{C}}\right)^{t-1}\left(\rho+{\overline{C}}\right)\|x_0\|\\
\leq &C_F  \left(\rho+{C_F\overline{C}}\right)^{t}\|x_0\|.
\end{align*}
\end{proof}

Next, based on Lemma~\ref{lemma:model_based_exponential_stability}, we consider the stability of the convex-combined policy $\pi=\lambda\widehat{\pi}+(1-\lambda)\overline{\pi}$ where $\overline{\pi}$ is a model-free policy satisfying the $\varepsilon$-\textit{consistency} in Definition~\ref{def:epsilon_consistency} and $\overline{\pi}$ is a model-based policy. 

\begin{theorem}
\label{thm:optimal_modelbased_difference}
Let $\pi^*$ be an optimal policy and $\overline{\pi}(x)=-Kx$ be a linear model-based policy. It follows that for any $t\geq 0$, $\|\pi_t^*(x)-\overline{\pi}(x)\|\leq C^{\mathsf{sys}}_{a} C_{\ell}\|x\|$ for some constant $C^{\mathsf{sys}}_{a}>0$ where $C_{\ell}$ is the Lipschitz constant  defined in Assumption~\ref{assumption:continuity} and
\begin{align}
\nonumber
C^{\mathsf{sys}}_{a}\coloneqq & 2\|R+B^\top PB\|^{-1}\Big(\|PF\|+(1+\|K\|)\left(\|PB\|+\|P\|\right)\\
  \nonumber
 & \quad +\frac{C^{\mathsf{sys}}_{b}}{2}\|B+ I\|(1+\|F\|+\|K\|)\Big),\\
  \nonumber
C^{\mathsf{sys}}_{b}\coloneqq & \frac{2{{C_F}^2}\|P\|(\rho+\overline{C})\left(\rho+(1+\|K\|)\right)}{1-(\rho+\overline{C})^2}\sqrt{\frac{\|Q+K^\top RK\|}{\sigma}}.
\end{align}
\end{theorem}

\begin{proof}
We use $\pi^*_t(x)=-Kx+h_t(x)$ to characterize an optimal policy at each time $t\geq 0$. Consider the following Bellman optimality equation:
\begin{align}
   V_t(x) =  \min_u \{ x^\top Q x + u_t^\top R u_t +  V_{t+1}(Ax + Bu + f_t(x,u))\}
\end{align}
where $V_t:\mathbb{R}^{\nn}\rightarrow\mathbb{R}_+$ denotes the optimal value function.
Using lemma \ref{lemma:model_based_exponential_stability}, \begin{align}
\nonumber
    V_t(x)\le& x^\top (Q+K^\top R K)x + V_{t+1}(F^\top x+f_t(x,u))\\
    \nonumber
    \le& V_{\infty}(0)+ \|Q+K^\top RK\|\sum_{t=0}^\infty {{C_F}^2}(\rho+C_F\overline{C})^{2t}\|x\|^2\\
    \label{eq:value_bound}
    =&V_{\infty}(0)+\frac{{{C_F}^2}\|Q+K^\top RK\|}{1-(\rho+C_F\overline{C})^2}\|x\|^2.
\end{align}
Let $x_{t}^{(\tau)}$ be the $(t+\tau)$-th state when applying optimal control with $x_t=x$ and write $u_t^{(\tau)}=\pi_{t+\tau}^*(x_t^{(\tau)})$ for all $t\geq 0$. Rearranging the terms in~\eqref{eq:value_bound}, it follows that
\begin{align*}
\frac{{{C_F}^2}\|Q+K^\top RK\|}{1-(\rho+C_F\overline{C})^2}\|x\|^2\geq & V_t(x)-V_{(\infty)}(0)
 = \sum_{\tau=t}^\infty \left(x_t^{(\tau-t)}\right)^{\top} Q x_t^{(\tau-t)}+\left(u_t^{(\tau-t)}\right)^\top Ru_t^{(\tau-t)} \\
    \geq&\lambda_{\min}(Q)\sum_{\tau=0}^\infty\|x_t^{(\tau)}\|^2+\lambda_{\min}(R)\sum_{\tau=0}^\infty \|u_t^{(\tau)}\|^2.
\end{align*}
Write $V_{t}(x)=x^\top P x+g_t(x)$ with $P$ denoting the solution of the Riccatti equation in~\eqref{eq:dare}. Then
\begin{align*}
    & x^\top P x + g_t(x) \\
    &= \min_{u}\Big[ x^\top Q x + u^\top R u + (Ax+Bu)^\top P(Ax+Bu)\\
    & \quad +2(Ax+Bu)^\top P f_t(x,u) + f_t(x,u)^\top P f_t(x,u) +  g_{t+1}(Ax+Bu+f_t(x,u)) \Big]\\
    &= \min_{u}\Big[ x^\top (Q + A^\top PA) x + u^\top (R+B^\top P B) u + 2 u^\top B^\top PAx  +2(Ax+Bu)^\top P f_t(x,u)\\
    &\quad + f_t(x,u)^\top P f_t2(x,u)+ g_{t+1}(Ax+Bu+f_t(x,u)) \Big]\\
    &= \min_{u}\Big[ x^\top (Q + A^\top PA - A^\top PB (R+B^\top PB)^{-1}B^\top PA ) x \\
    &\quad + (u+ (R+B^\top P B))^{-1}B^\top PAx )^\top (R+B^\top P B) (u+(R+B^\top PB)^{-1}B^\top PAx) \\
    &\quad+2(Ax+Bu)^\top P f_t(x,u) + f_t(x,u)^\top P f_t(x,u)+ g_{t+1}(Ax+Bu+f_t(x,u)) \Big].
\end{align*}
Since $P$ is the solution of the DARE in~\eqref{eq:dare}, letting $u=-Kx+v$ and $F\coloneqq A-BK$,
\begin{align*}
    g_t(x) &= \min_{v} \Big[  v^\top (R+ B^\top PB)v + 2  x^\top F^{\top} P f_t(x,-Kx+v) + 2 v^\top  B^\top P f_t(x,-Kx+v)  \\
    &\quad + f_t(x,-Kx+v)^\top P f_t(x,-Kx+v) + g_{t+1}(F x+ Bv+ f_t(x,-Kx+v)) \Big].
\end{align*}
Denoting by $v_t^*$ an optimal solution,
\begin{align*}
     g_t(x)&= (v_t^*)^\top (R+ B^\top PB)v_t^* + 2  x^\top F^\top P f_t(x,-Kx+v_t^*) + 2 (v_t^*)^\top  B^\top P f_t(x,-Kx+v_t^*)  \\
     &\quad + f_t(x,-Kx+v_t^*)^\top P f_t(x,-Kx+v_t^*) + g_{t+1}\left(F x+ Bv_t^*+ f_t(x, -Kx+v_t^*)\right).
\end{align*}
Denote by $\nabla g_t$ the Jacobian of $g_t$. We obtain \begin{align*}
    \nabla g_t(x)&= 2(R+B^\top PB)v_t^*\nabla v_t^*\\
    &\quad +2F^{\top}Pf_t(x,-Kx+v_t^*)\\
    &\quad +2x^\top F^{\top}P\nabla f_t(x,-Kx+v_t^*)[I,-K+\nabla v_t^*]\\
    &\quad +2\nabla (v_t^*)^\top B^\top Pf_t(x,-Kx+v_t^*)\\ &\quad +2(v_t^*)^\top B^\top P\nabla f_t(x,-Kx+v_t^*)[I,-K+\nabla v_t^*]\\
    &\quad +2Pf_t(x,-Kx+v_t^*)\nabla f_t(x,-Kx+v_t^*)[I,-K+\nabla v_t^*]\\
    &\quad +\nabla g_{t+1}(Fx+Bv_t^*+f_t(x,-Kx+v_t^*))(F+B\nabla v_t^*+\nabla f_t(x,-Kx+v_t^*)[I,-K+\nabla v_t^*])
\end{align*}
Noting that $v_t^*$ is a minimizer, the Jacobian of $g_t$ with respect to $v$ takes zero at $v=v_t^*$: \begin{equation}
\label{eq:gradient_of_v}
    \begin{split}
    &2(R+B^\top PB)v+2x^\top F^{\top}P\nabla f_t(x,-Kx+v)[0,I]+2B^\top Pf_t(x,-Kx+v)\\
    &+2v^\top B^\top P\nabla f_t(x,-Kx+v)[0,I]+2Pf_t(x,-Kx+v)\nabla f_t(x,-Kx+v)[0,I]\\
    &+(B+\nabla f_t(x,-Kx+v)[0,I])^\top\nabla g_{t+1}(F x+Bv+f_t(x,-Kx+v))|_{v=v_t^*}=0
    \end{split}
\end{equation}
Substituting above into the Jacobian of $g_t$, we get \begin{align*}
    \nabla g_t(x)&=2F^{\top}Pf_t(x,-Kx+v_t^*)\\
    &\quad +2x^\top F^{\top}P\nabla f_t(x,-Kx+v_t^*)[I,-K]\\
    & \quad +2(v_t^*)^\top B^\top P\nabla f_t(x,-Kx+v_t^*)[I,-K]\\
    & \quad +2Pf_t(x,-Kx+v_t^*)\nabla f_t(x,-Kx+v_t^*)[I,-K]\\
    &\quad +\nabla g_{t+1}(F x+Bv_t^*+f_t(x,-Kx+v_t^*))(F+\nabla f_t(x,-Kx+v_t^*)[I,-K])\\
    &= 2F^{\top}Pf_t(x,-Kx+v_t^*)+2(\nabla f_t(x,-Kx+v_t^*)[I,K])^\top Px\\
    &\quad +(F+\nabla f_t(x,-Kx+v_t^*)[I,-K])^\top\nabla g_{t+1}(x_t^{(1)}),
    \end{align*}
which implies that
\begin{align}  
\nonumber
      \nabla g_t(x)  &= \prod_{\tau=t}^\infty(F+\nabla f_{\tau}(x^{(\tau-t)}_{t},u^{(\tau-t)}_{t})[I,-K])^\top g_{\infty}(0)\\
      \nonumber
    &\quad + \sum_{\tau=t}^\infty\prod_{k=t}^{\tau}(F+\nabla f_{k}(x^{(k-t)}_t, u^{(k-t)}_t)[I,-K])^\top 2F^{^\top}Pf_{\tau}(x^{(\tau-t)}_{t},u^{(\tau-t)}_{t})\\
    \label{eq:term_proof_thm4}
    &\quad +\sum_{\tau=t}^\infty\prod_{k=t}^{\tau}(F+\nabla f_{k}(x_t^{(k-t)}, u_t^{(k-t)})[I,-K])^\top 2(\nabla f_{\tau}(x^{(\tau-t)}_{\tau},u^{(\tau-t)}_{\tau})[I,-K])^\top P^*_{\tau+1}.
\end{align}
Note that for any sequence of pairs $\left((x_t^{(\tau)},u_t^{(\tau)}):\tau\geq 0\right)$, \begin{align}
\nonumber
    &\left\|\prod_{k=t}^{\tau}F+\nabla f_k(x_t^{(k-t)},u_t^{(k-t)}[I,-K]\right\|\\
    \nonumber
    \leq &\sum_{k=t}^{\tau+1}\left\|F^{\tau+1-k}\sum_{\mathsf{S}\subseteq \{t,\ldots,\tau\},\left|\mathsf{S}\right|=k-t}\prod_{s\in \mathsf{S}}\nabla f_s(x_t^{(s-t)},u_t^{(s-t)})[I,-K]\right\|\\
    \nonumber
    \leq & \sum_{k=t}^{\tau+1} C_F\rho^{\tau+1-k}\sum_{\mathsf{S}\subseteq \{t,\ldots,\tau\},\left|\mathsf{S}\right|=k-t}\prod_{s\in \mathsf{S}}\left\|\nabla f_s(x_t^{(s-t)},u_t^{(s-t)})[I,-K]\right\|.
\end{align}
Since the residual functions are Lipschitz continuous as in Assumption~\ref{assumption:continuity}, their Jacobians satisfy $\|\nabla f_s(x_t^{(s-t)},u_t^{(s-t)})\|\leq C_{\ell}$ for any $x_t^{(s-t)}$ and $u_t^{(s-t)}$. Letting $\overline{C}\coloneqq C_{\ell} (1+\|K\|)$, we get
\begin{align}
\nonumber
&\left\|\prod_{k=t}^{\tau}F+\nabla f_k(x_t^{(k-t)},u_t^{(k-t)})[I,-K]\right\|\\
\label{eq:term_1_proof_thm4}
    \leq &\sum_{k=t}^{\tau+1} C_F\rho^{\tau+1-k}\sum_{\mathsf{S}\subseteq \{t,\ldots,\tau\},\left|\mathsf{S}\right|=k-t}\prod_{s\in \mathsf{S}}\overline{C}
    =C_F(\rho+\overline{C})^{\tau+1-t}.
\end{align}
Therefore, using~\eqref{eq:term_1_proof_thm4},
\begin{align}
\nonumber
    &\left\|\nabla f_{\tau}(x_t^{(\tau+1t-)},u_t^{(\tau+1-t)})[I,-K]\prod_{k=t}^{\tau}F+\nabla f_k(x_t^{(k-t)},u_t^{(k-t)})[I,-K]\right\|\\
    \nonumber
    \leq &\left\|\nabla f_{\tau}(x_t^{(\tau+1t-)},u_t^{(\tau+1-t)})[I,-K]\right\|\left\|\prod_{k=t}^{\tau}F+\nabla f_k(x_t^{(k-t)},u_t^{(k-t)})[I,-K]\right\|\\
    \label{eq:term_2_proof_thm4}
    \leq& C_F\overline{C}(\rho+\overline{C})^{\tau+1-t}.
\end{align}
Combing~\eqref{eq:term_1_proof_thm4} and~\eqref{eq:term_2_proof_thm4} with~\eqref{eq:term_proof_thm4},
\begin{align}
\nonumber
   & \|\nabla g_t(x)\|\leq  \left\|\prod_{\tau=t}^\infty(F+\nabla f_{\tau}(x^{(\tau-t)}_{t},u^{(\tau-t)}_{t})[I,-K])^\top g_{\infty}(0)\right\|\\
    \nonumber
    &+ \underbrace{\left\|\sum_{\tau=t}^\infty\prod_{k=t}^{\tau}(F+\nabla f_{k}(x_t^{(k-t)}, u_t^{(k-t)})[I,-K])^\top 2F^{\top}Pf_{\tau}(x^{(\tau-t)}_t,u_t^{(\tau-t)})\right\|}_{(a)}\\
    \nonumber
    & +\underbrace{\left\|\sum_{\tau=t}^\infty\prod_{k=t}^{\tau}(F+\nabla f_{k}(x_t^{(k-t)}, u_t^{(k-t)})[I,-K])^\top 2(\nabla f_{\tau}(x_t^{(\tau-t)},u_t^{(\tau-t)})[I,-K])^\top P x_t^{(\tau+1-t)}\right\|}_{(b)}.
\end{align}
Since $g_{\infty}(0)=0$, the first term in the inequality above is $0$. The second term satisfies
\begin{align}
\nonumber
   (a)\leq & 2C_FC_\ell \|P\|\rho\sum_{\tau=t}^\infty (\rho+\overline{C})^{\tau+1-t}\|(x_t^{(\tau-t)},u_t^{(\tau-t)})\|\\
   \nonumber
   = & 2C_FC_\ell\|P\|\rho\sum_{\tau=0}^\infty(\rho+\overline{C})^{\tau+1}\|x_t^{(\tau)},u_t^{(\tau)}\|\\
   \label{eq:term_a_proof_thm4}
    \leq & 2C_F C_\ell\|P\|\rho\sqrt{\sum_{\tau=1}^\infty(\rho+\overline{C})^{2\tau}}\sqrt{\sum_{\tau=0}^\infty\left\|x_t^{(\tau)}\right\|^2+\left\|u_t^{(\tau)}\right\|^2}\\
    \nonumber
    \leq & 2C_FC_\ell\|P\|\rho\sqrt{\frac{(\rho+\overline{C})^2}{1-(\rho+\overline{C})^2}}\sqrt{\frac{{{C_F}^2}\|Q+K^\top RK\|}{\min\left\{\lambda_{\min}(Q),\lambda_{\min}(R)\right\}(1-(\rho+\overline{C})^2)}\|x\|^2}\\
    \label{eq:term_a_2_proof_thm4}
    =&\frac{2{{C_F}^2}\|P\|\rho C_\ell(\rho+\overline{C})}{1-(\rho+\overline{C})^2}\sqrt{\frac{\|Q+K^\top RK\|}{\sigma}}\|x\|
\end{align}
where we have used the Cauchy to derive~\eqref{eq:term_a_proof_thm4}; $\lambda_{\min}(Q)$ and $\lambda_{\min}(R)$ are smallest eigenvalues of the matrices $Q$ and $R$ respectively and~\eqref{eq:term_a_2_proof_thm4} follows from Assumption~\ref{assumption:stability}.
Similarly, the third term satisfies
\begin{align}
\nonumber
    (b)\leq & 2C_F\overline{C}\|P\|\sum_{\tau=t}^\infty (\rho+\overline{C})^{\tau+1-t}\left\|x_t^{(\tau+1-t)}\right\|\\
    \nonumber
    \leq & 2C_F\overline{C}\|P\|\sqrt{\sum_{\tau=1}^\infty(\rho+\overline{C})^{2\tau}}\sqrt{\sum_{\tau=1}^\infty\left\|x_t^{(\tau)}\right\|^2}\\
    \nonumber
    \leq & 2C_F\overline{C}\|P\|\sqrt{\frac{(\rho+\overline{C})^2}{1-(\rho+\overline{C})^2}}\sqrt{\frac{{{C_F}^2}\|Q+K^\top RK\|}{\lambda_{\min}(Q)(1-(\rho+\overline{C})^2}\|x\|^2}\\
    \label{eq:term_b_proof_thm4}
    \leq & \frac{2{{C_F}^2}\|P\|\overline{C}(\rho+\overline{C})}{1-(\rho+\overline{C})^2}\sqrt{\frac{\|Q+K^\top RK\|}{\sigma}}\|x\|.
\end{align}
Putting~\eqref{eq:term_a_2_proof_thm4} and~\eqref{eq:term_b_proof_thm4} together, we conclude that 
\begin{align}
\label{eq:proof_csineq}
\left\|\nabla g_t(x)\right\|\leq &  \frac{2{{C_F}^2}\|P\|(\rho+\overline{C})\left(\rho C_\ell+\overline{C}\right)}{1-(\rho+\overline{C})^2}\sqrt{\frac{\|Q+K^\top RK\|}{\sigma}}\|x\|
  =: C_\nabla C_{\ell}\|x\|.
\end{align}
Rewriting \eqref{eq:gradient_of_v} as 
\begin{align*}
    -v_t^*& =
    (R+B^\top PB)^{-1}x^\top F^{\top}P\nabla f_t(x,-Kx+v_t^*)[0,I]\\
    &\ +(R+B^\top PB)^{-1}B^\top Pf_t(x,-Kx+v_t^*)\\
    &\ +(R+B^\top PB)^{-1}(v_t^*)^\top B^\top P\nabla f_t(x,-Kx+v_t^*)[0,I]\\
    &\ +(R+B^\top PB)^{-1}Pf_t(x,-Kx+v_t^*)\nabla f_t(x,-Kx+v_t^*)[0,I]\\
    &\ +\frac{1}{2}(R+B^\top PB)^{-1}(B+\nabla f_t(x,-Kx+v_t^*)[0,I])^\top\nabla g_t(F x+Bv+f_t(x,-Kx+v_t^*))
\end{align*}
and taking the Euclidean norm on both sides, we obtain 
\begin{align}
\nonumber
    \|v_t^*\|\leq & \left\|(R+B^\top PB)^{-1}x^\top F^{\top}P\nabla f_t(x,-Kx+v_t^*)[0,I]\right\|\\
    \nonumber
    &\quad + \left\|(R+B^\top PB)^{-1}B^\top Pf_t(x,-Kx+v_t^*)\right\|\\
    \nonumber
    &\quad +\left\|(R+B^\top PB)^{-1}(v_t^*)^\top B^\top P\nabla f_t(x,-Kx+v_t^*)[0,I]\right\|\\
    \nonumber
    &\quad +\left\|(R+B^\top PB)^{-1}Pf_t(x,-Kx+v_t^*)\nabla f_t(x,-Kx+v_t^*)[0,I]\right\|\\
    \nonumber
    &\quad +\Big\|\frac{1}{2}(R+B^\top PB)^{-1}(B+\nabla f_t(x,-Kx+v_t^*)[0,I])^\top\\
    \nonumber
    &\qquad\qquad \nabla g_t\left(F x+Bv+f_t(x,-Kx+v_t^*)\right)\Big\|\\
    \nonumber
    \leq &\left\|R+B^\top PB\right\|^{-1}\Big(C_\ell\|PF\|\|x\|+\overline{C}\left\|B^\top P\right\|\|x\|\\
    \nonumber
    &\quad +2C_\ell\left\|B^\top P\right\|\|v_t^*\|+\overline{C}\|P\|\|x\|+C_\ell\|P\|\|v_t^*\|\\
    \label{eq:term_final_proof_thm4}
    &\quad + \frac{1}{2}C_\nabla C_\ell \|B+C_\ell I\|(\|F\|\|x\|+\|B\|\|v_t^*\|+\overline{C}\|x\|+C_\ell\|x\|)\Big).
\end{align}
Finally, assuming $C_\ell\leq\min\left\{1,\frac{\|R+B^\top PB\|}{4\left\|B^\top P\right\|+2\|P\|+C_\nabla(\|B\|+1)\|B\|}\right\}$,~\eqref{eq:term_final_proof_thm4} yields
\begin{align*}
    &\left\|\pi_t^*(x)-\overline{\pi}(x)\right\|\leq C_\ell\|x\|\times\\
    & \underbrace{\left(2{\|R+B^\top PB\|^{-1}\left((\|PF\|+(1+\|K\|)\left(\|PB\|+\|P\|\right)+\frac{C_\nabla}{2}\|B+ I\|(2+\|F\|+\|K\|)\right)}\right)}_{ =: C^{\mathsf{sys}}_{a}}
\end{align*}
where the constant $C_\nabla$ is defined as
\begin{align*}
    C_\nabla\coloneqq \frac{2{{C_F}^2}\|P\|(\rho+\overline{C})\left(\rho+(1+\|K\|)\right)}{1-(\rho+\overline{C})^2}\sqrt{\frac{\|Q+K^\top RK\|}{\sigma}}.
\end{align*}
\end{proof}

\begin{theorem}
\label{thm:blackbox_stability}
Let $\gamma\coloneqq  \left(\rho+{C_F C_{\ell}(1+\|K\|)}\right)$.
Suppose the black-box policy $\widehat{\pi}$ is $\varepsilon$-consistent with
the consistency constant $\varepsilon$ satisfying $\varepsilon<\frac{1/{C_F}-C^{\mathsf{sys}}_{a}C_{\ell}}{C_{\ell}+\|B\|}$.
Suppose the Lipschitz constant $C_{\ell}$ satisfies $
 C_{\ell}<\min\left\{1, 1/({C_F} C_{a}^{\mathsf{sys}}),  C_{c}^{\mathsf{sys}} ,{(1-\rho)}/{({C_F}(1+\|K\|))}\right\}$. 
Let $(x_t: t\geq 0)$ denote a trajectory of states generated by the adaptive $\lambda$-confident policy $\pi_t=\lambda_t \widehat{\pi} + (1-\lambda_t)\overline{\pi}$ (Algorithm~\ref{alg:adaptive_policy}). Then it follows that  $\pi_t$ is an exponentially stabilizing policy such that
$
\|x_t\|\leq   \frac{\gamma^t-\mu^t}{1-\mu\gamma^{-1}}\left({C_F}+\mu\gamma^{-1}\right)\|x_0\|
$ for any $t\geq 0$
where $\mu\coloneqq {C_F}\left(\varepsilon\left(C_{\ell}+\|B\|\right)+C^{\mathsf{sys}}_{a} C_{\ell}\right)$.
\end{theorem}
\begin{proof}
We first introduce a new symbol $x_t^{(\tau)}$, which is the $t$-th state of a trajectory generated by the combined policy $\pi_t(x)=\lambda_t\widehat{\pi}(x)+(1-\lambda_t)\overline{\pi}(x)$ for the first $\tau$ steps and then switch to the model-based policy $\overline{\pi}$ for the remaining steps. Let $(\overline{x}_t: t\geq 0)$ be a trajectory of states generated by a model-based policy $\overline{\pi}$. As illustrated in Figure~\ref{fig:fixed_lambda}, for any $t\geq 0$,
\begin{align}
\label{eq:proof_telescope}
     x_t-\overline{x}_t&=\sum_{\tau=0}^{t-1}x_t^{(\tau+1)}-x_t^{(\tau)}.
\end{align}
\begin{figure}
    \centering
\includegraphics[scale=0.2]{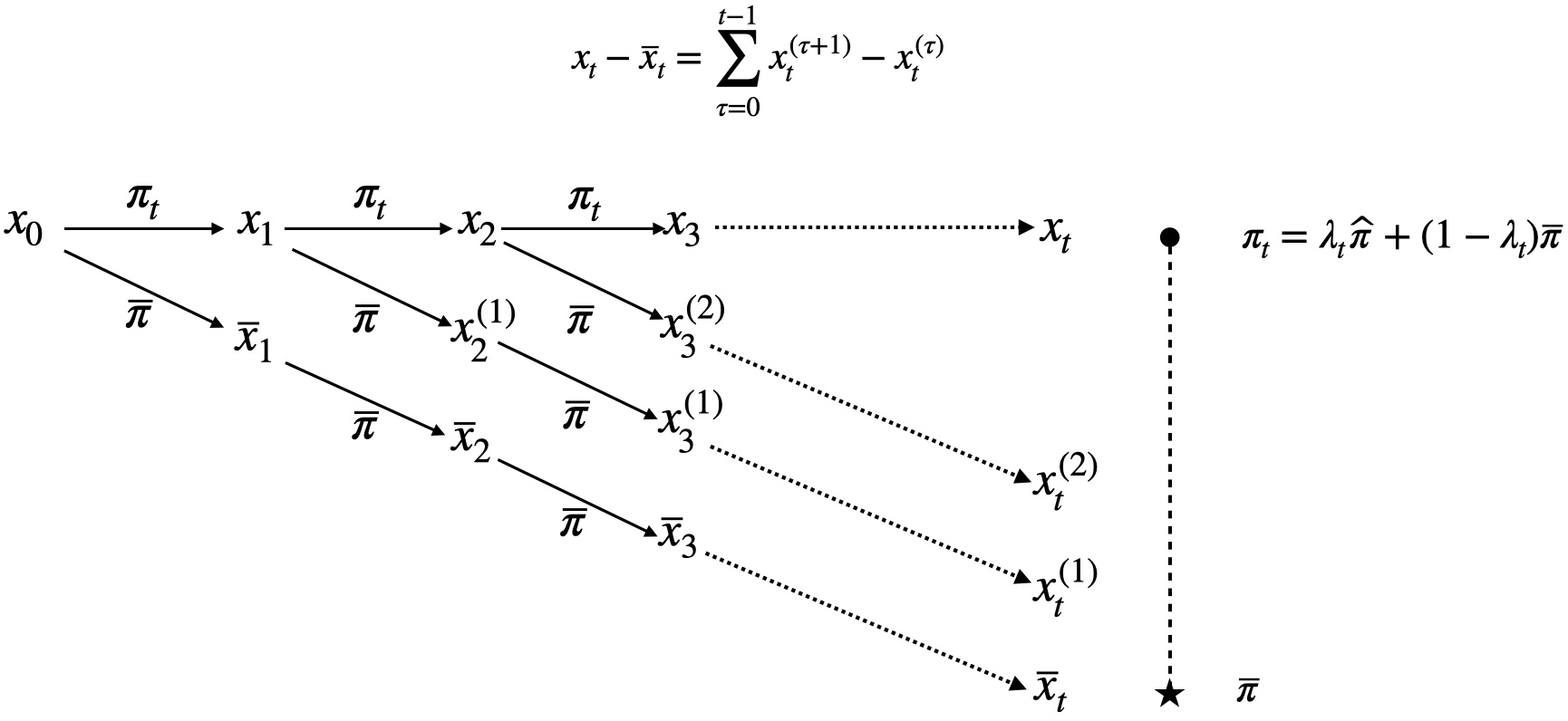}
\caption{Telescoping sum of $x_t-x_t^*$.}
\label{fig:telescoping}
\end{figure}

Let $(x_t: t\geq 0)$ and $(x'_t: t\geq 0)$ denote the trajectories of states at time $t$ generated by the model-based policy $\widehat{\pi}(x)=-Kx$ when the initial states are $x_0$ and $x_0'$ respectively. Then~\eqref{eq:stability_proof_1} leads to
\begin{align*}
     x_t-x'_t=F^{t}(x_0-x'_0)+\sum_{\tau=0}^{t-1}F^{t-1-\tau} \left(f_{\tau}(x_\tau,-Kx_\tau)-f_{\tau}(x'_\tau,-Kx'_\tau)\right),
\end{align*}
yielding
\begin{align*}
    \|x_t-x'_t\| \leq  C_F\rho^t {\left(\left\|x_0-x'_0\right\|+C_{\ell} (1+\|K\|)\sum_{\tau=0}^{t-1}\rho^{-1-\tau}\left\|x_{\tau}-x'_{\tau}\right\|\right)}
\end{align*}
where we have used the Lipschitz continuity of $(f_t:t\geq 0)$ and Assumption~\ref{assumption:stability} so that $\left\| F^t\right\|\leq C_F\rho^t$ for any $t\geq 0$. The same argument as in Lemma~\ref{lemma:model_based_exponential_stability} gives that for any $t\geq 0$,
\begin{align*}
 \|x_t-x'_t\| \leq C_F  \left(\rho+{C_F\overline{C}}\right)^{t}\|x_0-x'_0\|
\end{align*}
where $\overline{C}\coloneqq C_{\ell} (1+\|K\|)$. Continuing from~\eqref{eq:proof_telescope}, $x_t-x_t^*$ can be represented by a telescoping sum (illustrated in Figure~\ref{fig:telescoping}):
\begin{align*}
\left\|x_t-\overline{x}_t\right\|&=\left\|\sum_{\tau=0}^{t-1}x_t^{(\tau+1)}-x_t^{(\tau)}\right\|\leq\sum_{\tau=0}^{t-1}\left\|x_t^{(\tau+1)}-x_t^{(\tau)}\right\|\leq \sum_{\tau=0}^{t-1}C_F\gamma^{t-\tau-1}\left\|x_{\tau+1}^{(\tau+1)}-x_{\tau+1}^{(\tau)}\right\|.
\end{align*}
For any $t\geq 0$,
\begin{align*}
&x_t^{(t)}-x_t^{(t-1)} = x_t-x_t^{(t-1)} \\
=& (x_t - \lambda_t x_t') + \left(\lambda_t x_t'-x_t^{(t-1)}\right)\\
=& \lambda_t\left(\widehat{x}_t - x_t'\right) + (1-\lambda_t)\left(\overline{x}_t-x_t^{(t-1)}\right) + \lambda_t\left(x_t'-x_t^{(t-1)}\right)\\
=& \lambda_t\left(\widehat{x}_t - x_t'\right) + \lambda_t\left(x_t'-x_t^{(t-1)}\right)
\end{align*}
where $x_t'$, $\widehat{x}_t$ and $\overline{x}_t$ are the states generated by running an optimal policy $\pi^*$, a model-free policy $\widehat{\pi}$ and a model-based policy $\overline{\pi}$ respectively for one step with the same initial state $x_{t-1}$. Note that $x_t=\lambda_t\widehat{x}_t + (1-\lambda_t)\overline{x}_t$ and $\overline{x}_t=x_t^{(t-1)}$. Therefore,
\begin{align*}
 \lambda_t\left(\widehat{x}_t - x_t'\right)
= &\lambda_t \big( \underbrace{B(\widehat{\pi}(x_{t-1})-{\pi}^*(x_{t-1}))}_{:=(a)}+\underbrace{f_{t-1}(x_{t-1},\widehat{\pi}(x_{t-1}))-f_{t-1}(x_{t-1},{\pi}^*(x_{t-1}))}_{:=(b)}\big).
\end{align*}
For (a), we obtain the following bound:
\begin{align}
   \label{eq:stability_proof_4}
   \| (a)\|\leq \|B\|\|\widehat{\pi}(x_{t-1})-{\pi}^*(x_{t-1})\|
   \leq & \varepsilon\|B\|\|x_{t-1}\|
\end{align}
and~\eqref{eq:stability_proof_4} holds since the model-free policy $\widehat{\pi}$ is $\varepsilon$-consistent (Definition~\ref{def:epsilon_consistency}). Similarly, for (b), since the functions $(f_t:t\geq 0)$ are  Lipschitz continuous (Assumption~\ref{assumption:continuity}),
\begin{align}
   \label{eq:stability_proof_5}
    \| (b)\|\leq  \varepsilon C_{\ell}\|x_{t-1}\|.
\end{align}
Applying Theorem~\ref{thm:optimal_modelbased_difference},
$
    \left\|x_t'-x_t^{(t-1)}\right\|\leq C^{\mathsf{sys}}_{a} C_{\ell}\|x_{t-1}\|.
$
Combining~\eqref{eq:stability_proof_4} and~\eqref{eq:stability_proof_5} and applying Lemma~\ref{lemma:model_based_exponential_stability},
$
\left\|x_t-\overline{x}_t\right\| \leq  \mu\sum_{\tau=0}^{t-1}\gamma^{t-\tau-1} \|x_{\tau}\|
$
where $\mu\coloneqq {C_F}\left(\varepsilon\left(C_{\ell}+\|B\|\right)+C^{\mathsf{sys}}_{a} C_{\ell}\right)$, therefore,
\begin{align*}
    \frac{\|x_t\|}{\gamma^{t}}\leq \frac{1-(\mu\gamma^{-1})^{t}}{1-\mu\gamma^{-1}}\left({C_F}+\mu\gamma^{-1}\right)\|x_0\|.
\end{align*}
Hence, if $\mu<1$,
$
\|x_t\|\leq   \frac{\gamma^t-\mu^t}{1-\mu\gamma^{-1}}\left({C_F}+\mu\gamma^{-1}\right)\|x_0\|,
$
the linearly combined policy  $\pi_t$ is an exponentially stabilizing policy.
\end{proof}



\subsection{Proof of Theorem~\ref{thm:stability}}


To show the stability results, noting that the policy is switched to the model-based policy after $t\geq t_0$. Since the Lipschitz constant $C_{\ell}$ satisfies $C_{\ell}<\frac{1-\rho}{C_F(1+\|K\|)}$ where $\left\| F^t\right\|\leq C_F\rho^t$ for any $t\geq 0$, then the model-based policy $\overline{\pi}$ is exponentially stable as shown in Lemma~\ref{lemma:model_based_exponential_stability}. Let $\mu\coloneqq C_F\left(\varepsilon\left(C_{\ell}+\|B\|\right)+C^{\mathsf{sys}}_{a} C_{\ell}\right)$ and $\gamma\coloneqq\rho+C_F C_{\ell}(1+\|K\|)<1$. Applying Theorem~\ref{thm:blackbox_stability} and Lemma~\ref{lemma:model_based_exponential_stability}, for any $t\geq 0$,
\begin{align*}
     \|x_t\|\leq  \frac{\mu^{t_0}}{\mu\gamma^{-1}-1}\left(C_F+\mu\gamma^{-1}\right)(\rho+C_F C_{\ell}(1+\|K\|))^{t-t_0}\|x_0\|.
    \end{align*}
Since $\lambda_{t+1}<\lambda_{t}-\alpha$ for all $t\geq 0$ with some $\alpha>0$, $t_0$ is finite and the adaptive $\lambda$-confident policy is an exponentially stabilizing policy.


\section{Competitive Ratio Analysis}
\label{app:proof_competitiveness}

\subsection{Proof of Theorem~\ref{thm:competitive}}



Before proceeding to the proof of Theorem~\ref{thm:competitive}, we prove the following lemma that will be useful.

\begin{lemma}
\label{lemma:competitive}
Suppose $\varepsilon\leq {\lambda_{\min}(Q)}/{(2\|H\|)}$. If the adaptive $\lambda$-confident policy  $\pi_t=\lambda_t \widehat{\pi} + (1-\lambda_t)\overline{\pi}$ (Algorithm~\ref{alg:adaptive_policy}) is an exponentially stabilizing policy, then the competitive ratio of the linearly combined policy is $       \mathsf{CR}(\varepsilon) =  O((1-{{\lambda}}){\overline{\mathsf{CR}}_{\mathrm{model}}})+O\left({1}/\left({{1-\frac{2\|H\|}{\sigma}}\varepsilon}\right)\right) + O(C_{\ell}\|x_0\|)$.
\end{lemma}

\begin{proof}[Proof of Lemma~\ref{lemma:competitive}]
Let $H\coloneqq R+B^\top P B$.
Fix any sequence of residual functions $(f_t:t\geq 0)$.
Denote by $f_t\coloneqq f_t(x_t,u_t)$, $f_t^*\coloneqq f_t(x^*_t,u^*_t)$ and $x^*_t$ and $u_t^*$ the offline optimal state and action at time $t$. With $u_t = \lambda_t \widehat{u}_t + (1-\lambda_t) \overline{u}_t$ at each time $t$, Lemma~\ref{lemma:gap} implies that the \textit{dynamic regret} can be bounded by
\begin{align}
\nonumber
\mathsf{DynamicRegret} & \coloneqq \mathsf{ALG} - \mathsf{OPT} 
\leq \sum_{t=0}^{\infty}\eta_t^{\top} H \eta_t  + O(1)\\
\nonumber
& + 2\sum_{t=0}^{\infty}\eta_t^{\top}B^{\top}\left({\sum_{\tau=t}^{\infty}}\left(F^{\top}\right)^{\tau}P(f_{t+\tau}-f^*_{t+\tau})\right)\\
\nonumber
& + \sum_{t=0}^{\infty}\left(f_t^{\top}Pf_t -\left(f_t^*\right)^{\top}Pf_t^* \right) 
+ 2x_0^{\top}\left(\sum_{t=0}^{\infty}\left(F^{\top}\right)^{t+1}P\left(f_{t}-f^*_{t}\right)\right)
\\
\nonumber
& + 2\sum_{t=0}^{\infty}\left(f_t^{\top}{\sum_{\tau=0}^{\infty}}\left(F^{\top}\right)^{\tau+1}Pf_{t+\tau+1} -\left(f_t^*\right)^{\top}{\sum_{\tau=0}^{\infty}}\left(F^{\top}\right)^{\tau+1}P\left(f_{t+\tau+1}^*\right)\right) \\
\label{eq:1.0}
& + 2\sum_{t=0}^{\infty}\left({\sum_{\tau=t}^{\infty}}\left(F^{\top}\right)^{\tau}Pf^*_{t+\tau}\right)BH^{-1}B^{\top}\left({\sum_{\tau=t}^{\infty}}\left(F^{\top}\right)^{\tau}P(f^*_{t+\tau}-f_{t+\tau})\right).
\end{align}
Consider the auxiliary linear policy defined in~\eqref{eq:auxiliary}.
Provided with any state $x\in\mathbb{R}^{\nn}$, the linearly combined policy $\pi_t$ is given by
\begin{align}
\nonumber
\pi_t(x)&=\lambda_t\widehat{\pi}(x)+(1-\lambda_t)\overline{\pi}(x)=\pi'(x)+\lambda_t(\widehat{\pi}(x)-\pi'(x))+(1-\lambda_t)\left(\overline{\pi}(x)-\pi'(x)\right),
\end{align}
implying $\eta_t=\lambda_t(\widehat{\pi}(x_t)-\pi'(x_t))+(1-\lambda_t)\left(\overline{\pi}(x_t)-\pi'(x_t)\right)$ in~\eqref{eq:1.0}. Moreover, $\sum_{t=0}^{\infty}\eta_t^{\top} H\eta_t \leq \sum_{t=0}^{\infty}\|H\|\|\eta_t\|^2$, therefore, denoting by ${{\lambda}}\coloneqq\lim_{t\rightarrow\infty}\lambda_t$,
\begin{align}
\label{eq:quadratic_bound}
\sum_{t=0}^{\infty}\eta_t^{\top} H\eta_t\leq & 2\|H\|\left(  \sum_{t=0}^{\infty} \left\|\widehat{\pi}(x_t)-\pi'(x_t)\right\|^2
+ (1-{{\lambda}})\sum_{t=0}^{\infty}\left\|\overline{\pi}(x_t)-\pi'(x_t)\right\|^2\right),
\end{align}
where in~\eqref{eq:quadratic_bound} we have used the the Cauchy–Schwarz inequality.
Since the model-free policy $\widehat{\pi}$ is $\varepsilon$-consistent, it follows that
\begin{align}
\nonumber
\left\|\widehat{\pi}(x_t)-\pi'(x_t)\right\|^2\leq& 2\left\|{\pi}_t^*(x_t)-\pi'(x_t)\right\|^2+2\left\|\widehat{\pi}(x_t)-\pi_t^*(x_t)\right\|^2\\
\nonumber
\leq &2\left\|{\pi}_t^*(x_t)-\pi'(x_t)\right\|^2+2\varepsilon\|x_t\|^2.
\end{align}
Furthermore, since the cost $\mathsf{OPT'}$ induced by the auxiliary linear policy~\eqref{eq:auxiliary} is smaller than $\mathsf{OPT}$,
\begin{align}
\label{eq:proof_term1}
\sum_{t=0}^{\infty}\left\|{\pi}_t^*(x_t)-\pi'(x_t)\right\|^2\leq \frac{\mathsf{OPT}+\mathsf{OPT}'}{\lambda_{\min}(R)}\leq \frac{2\mathsf{OPT}}{\lambda_{\min}(R)}
\end{align}
where $\lambda_{\min}(R)>0$ denotes the smallest eigenvalue of $R\succ 0$.

The linear quadratic regulator is $\overline{\pi}(x)=-Kx=-(R+B^{\top}PB)^{-1}B^{\top}PAx$ for a given state $x\in\mathbb{R}^{\nn}$, we have for all $t\geq 0$,
\begin{align}
\label{eq:proof_term2}
\left\|\overline{\pi}(x_t)-\pi'(x_t)\right\|^2 = \left\| \sum_{\tau=t}^{\infty}\left(F^\top\right)^{\tau-t} P f_{\tau}^*\right\|^2.
\end{align}
Plugging~\eqref{eq:proof_term1} and~\eqref{eq:proof_term2} into~\eqref{eq:quadratic_bound},
\begin{align}
\label{eq:eta_bound}
\sum_{t=0}^{\infty}\eta_t^{\top} H\eta_t 
\leq & 2\|H\|\Bigg(\left( \frac{2\mathsf{OPT}}{\lambda_{\min}(R)}+ \varepsilon \sum_{t=0}^{\infty}\left\|x_t\right\|^2 \right)
+ \sum_{t=0}^{\infty}\Bigg\|\sum_{\tau=t}^{\infty}\left(F^\top\right)^{\tau-t} P f_{\tau}^*\Bigg\|^2\Bigg).
\end{align}
The algorithm cost $\mathsf{ALG}$ can be bounded by
$
    \mathsf{ALG} \geq \sum_{t=0}^{\infty}x_t^{\top}Qx_t+ u_t^{\top}Ru_t\geq \sum_{t=0}^{\infty}\lambda_{\min}(Q)\|x_t\|^2,
$
therefore,~\eqref{eq:eta_bound} leads to
\begin{align}
\label{eq:eta_bound2}
\sum_{t=0}^{\infty}\eta_t^{\top} H\eta_t 
&\leq 2\|H\|\Bigg(\left( \frac{2\mathsf{OPT}}{\lambda_{\min}(R)}+ \varepsilon   \frac{\mathsf{ALG}}{\lambda_{\min}(Q)}  \right)
+ \sum_{t=0}^{\infty}\Bigg\|\sum_{\tau=t}^{\infty}\left(F^\top\right)^{\tau-t} P f_{\tau}^*\Bigg\|^2\Bigg).
\end{align}

Moreover, we have the following lemma holds.

\begin{lemma}
\label{lemma:optimal_cost_lower_bound}
The optimal cost $\mathsf{OPT}$ can be bounded from below by $$
  \mathsf{OPT}\geq  \frac{D_0 (1-\rho)^2}{{{C_F}^2}\|P\|^2}  \sum_{t=0}^{\infty}\left\| \sum_{\tau=t}^{\infty}\left(F^\top\right)^{\tau-t} P f_{\tau}^*\right\|^2.
$$
\end{lemma}

\begin{proof}[Proof of Lemma~\ref{lemma:optimal_cost_lower_bound}]
the optimal cost can be bounded from below by
\begin{align}
\nonumber
  \mathsf{OPT}= & \sum_{t=0}^{\infty}(x_t^*)^{\top} Q x^*_t + (u_t^*)^{\top} R u^*_t\\
  \label{eq:4.60}
  \geq &\sum_{t=0}^{\infty} \lambda_{\min}(Q)\left\|x^*_t\right\|^2 + \lambda_{\min}(R)\|u^*_t\|^2\\
  \nonumber
  \geq & 2D_0 \sum_{t=0}^{\infty} \left(\|Ax^*_t\|^2+\|B u^*_t\|^2\right)+\frac{1}{2}\sum_{t=0}^{\infty}\lambda_{\min}(Q)\|x^*_{t}\|^2\\
  \nonumber
    \geq & D_0 \sum_{t=0}^{\infty} \left\|Ax^*_t+Bu^*_t\right\|^2+\frac{1}{2}\sum_{t=0}^{\infty}\lambda_{\min}(Q)\|x^*_{t}\|^2.
\end{align}
Since $x^*_{t+1}=Ax^*_t+Bu^*_t+f_t^*$ for all $t\geq 0$,
 \begin{align}   
\nonumber
  \mathsf{OPT}\geq & D_0\sum_{t=0}^{\infty}\left\|x^*_{t+1}-f_t^*\right\|^2+\frac{1}{2}\sum_{t=0}^{\infty}\lambda_{\min}(Q)\left\|x^*_{t}\right\|^2\\
\label{eq:4.80}
\geq & \frac{D_0}{2}\sum_{t=0}^{\infty}\left\|f_t^*\right\|^2  +\left(\frac{\lambda_{\min}(Q)}{2}-D_0\right)\sum_{t=0}^{\infty}\left\|x^*_t\right\|^2 
\end{align}
for some constant $D_0\coloneqq\min\{\frac{\lambda_{\min}(R)}{\|B\|},\frac{\lambda_{\min}(Q)}{2\|A\|},\frac{\lambda_{\min}(Q)}{2}\}\geq \frac{\sigma}{\max\{2,\|A\|,\|B\|\}}$ (Assumption~\ref{assumption:stability}) that depends on known system parameters $A,B,Q$ and $R$
where in~\eqref{eq:4.60}, $\lambda_{\min}(Q)$, $\lambda_{\min}(R)$ are the smallest eigenvalues of positive definite matrices $Q,R$, respectively. 
Let $\psi_t\coloneqq {\sum_{\tau=t}^{\infty}}\left(F^\top\right)^\tau P f_{t+\tau}^*$ for all $t\geq 0$.
Note that $F=A-BK$ and we define $\rho\coloneqq {(1+\rho(F))}/{2}<1$ where  $\rho(F)$ denotes the spectral radius of $F$. From the Gelfand’s formula, there exists a constant $C_F\geq 0$ such that $\|F^t\|\leq C_F\rho^{t}$ for all $t\geq 0$.
Therefore,
\begin{align}
\nonumber
 \sum_{t=0}^{\infty}\left\|\psi_t\right\|^2 \coloneqq&  \sum_{t=0}^{\infty}\left\|{\sum_{\tau=t}^{\infty}}\left(F^{\top}\right)^{\tau}P f_{t+\tau}^*\right\|^2
 \leq   {{C_F}^2} \|P\|^2\sum_{t=0}^{\infty}\left({\sum_{\tau=t}^{\infty}}\rho^\tau\left\|f_{t+\tau}^*\right\|\right)^2\\
 \nonumber
 = &  {{C_F}^2} \|P\|^2\sum_{t=0}^{\infty}{\sum_{\tau=t}^{\infty}}\sum_{\ell=0}^{\infty}\rho^{\tau}\rho^{\ell}\left\|f_{t+\tau}^*\right\|\left\|f_{t+\ell}^*\right\|\\
 \label{eq:4.70}
 \leq & \frac{{{C_F}^2}}{2} \|P\|^2\sum_{t=0}^{\infty}{\sum_{\tau=t}^{\infty}}\sum_{\ell=0}^{\infty}\rho^{\tau}\rho^{\ell}\left(\left\|f_{t+\tau}^*\right\|^2+\left\|f_{t+\ell}^*\right\|^2\right).
\end{align}
Continuing from~\eqref{eq:4.70},
\begin{align}
\nonumber
 \sum_{t=0}^{\infty}\left\|\psi_t\right\|^2\leq &\frac{{C_F}^2}{2} \|P\|^2\left(\sum_{\ell=0}^{\infty}\rho^{\ell}\right)\sum_{t=0}^{\infty}{\sum_{\tau=t}^{\infty}}\rho^{\tau}\left\|f_{t+\tau}^*\right\|^2 + \frac{{{C_F}^2}}{2} \|P\|^2\left({\sum_{\tau=t}^{\infty}}\rho^{\tau}\right)\sum_{t=0}^{\infty}\sum_{\ell=0}^{\infty}\rho^{\ell}\left\|f_{t+\ell}^*\right\|^2\\
 \nonumber
 \leq & \frac{{{C_F}^2}}{1-\rho} \|P\|^2\sum_{t=0}^{\infty}{\sum_{\tau=t}^{\infty}}\rho^{\tau}\left\|f_{t+\tau}^*\right\|^2
 \leq   \frac{{{C_F}^2}}{1-\rho} \|P\|^2\sum_{t=0}^{\infty}\sum_{\tau=0}^{\infty}\rho^{\tau}\left\|f^*_{t+\tau}\right\|^2\\
\nonumber
 = & \frac{{{C_F}^2}}{1-\rho} \|P\|^2\left(\sum_{\tau=0}^{\infty}\rho^{\tau}\right)\left(\sum_{t=0}^{\infty}\left\|f_{t}^*\right\|^2\right)\\
  \label{eq:4.90}
 \leq & \frac{{{C_F}^2}}{(1-\rho)^2} \|P\|^2 \sum_{t=0}^{\infty}\left\|f_t^*\right\|^2.
\end{align}
Putting~\eqref{eq:4.90} into~\eqref{eq:4.80}, we obtain
$
  \mathsf{OPT}\geq  \frac{D_0 (1-\rho)^2}{{{C_F}^2}\|P\|^2}  \sum_{t=0}^{\infty}\|\psi_t\|^2.
$
\end{proof}

Combining Lemma~\ref{lemma:optimal_cost_lower_bound} with~\eqref{eq:eta_bound2},
\begin{align}
\label{eq:eta_bound3}
\sum_{t=0}^{\infty}\eta_t^{\top} H\eta_t 
&\leq 2\|H\|\Bigg( \frac{2\mathsf{OPT}}{\lambda_{\min}(R)}+ \varepsilon   \frac{\mathsf{ALG}}{\lambda_{\min}(Q)}
+ (1-{{\lambda}})\frac{{{C_F}^2}\|P\|^2\mathsf{OPT}}{D_0 (1-\rho)^2}\Bigg).
\end{align}


Furthermore, since $P$ is symmetric, $f_t^\top P f_t-(f_t^*)^\top P f_t^*=\left(f_t+f_t^*\right)^\top P \left(f_t-f_t^*\right)$, the RHS of the inequality~\eqref{eq:1.0} can be bounded by
\begin{align}
\nonumber
&\sum_{t=0}^{\infty}\eta_t^{\top} H\eta_t + 2\sum_{t=0}^{\infty}\Bigg(\| B\eta\|\left\|{\sum_{\tau=t}^{\infty}}\left(F^{\top}\right)^{\tau}P(f_{t+\tau}-f_{t+\tau}^*)\right\| + \|P \left(f_t+f_t^*\right)\| \left\|f_t-f_t^*\right\|\\
\nonumber
& \quad +  \|x_0\|\left\|\sum_{t=0}^{\infty}\left(F^{\top}\right)^{t+1}P(f_{t}-f^*_{t})\right\| + \|f_t\|\left\|{\sum_{\tau=0}^{\infty}}\left(F^{\top}\right)^{\tau+1}P(f_{t+\tau+1}-f^*_{t+\tau+1})\right\|\\
\nonumber
& \quad +\left\|{\sum_{\tau=0}^{\infty}}\left(F^{\top}\right)^{\tau+1}Pf_{t+\tau+1}^*\right\|\left\|f_t^*-f_t\right\|\\ 
\nonumber
& \quad + \left\|BH^{-1}B^{\top}\right\|\left\|{\sum_{\tau=t}^{\infty}}\left(F^{\top}\right)^{\tau}Pf^*_{t+\tau}\right\|\left\|{\sum_{\tau=t}^{\infty}}\left(F^{\top}\right)^{\tau}P(f^*_{t+\tau}-f_{t+\tau})\right\|\Bigg).
\end{align}
Furthermore, by our assumption, the linearly combined policy $\pi$ is an exponentially stabilizing policy and since $f_t$ is Lipschitz continuous with a Lipschitz constant $C_{\ell}$, using~\eqref{eq:eta_bound3} and noting that  $\mathsf{OPT}>0$,
\begin{align*}  
\mathsf{ALG} - \mathsf{OPT} &\leq 2\|H\|\Bigg( \frac{1}{\sigma}\left(2\mathsf{OPT}+ \varepsilon   \mathsf{ALG}\right)
+ (1-{{\lambda}})\frac{{{C_F}^2}\|P\|^2{\max\{2,\|A\|,\|B\|\}}\mathsf{OPT}}{\sigma (1-\rho)^2}\Bigg)\\
&\quad +  O(C_{\ell}\|x_0\|).
\end{align*}
Rearranging  the terms gives the competitive ratio bound in Lemma~\ref{lemma:competitive}.
\end{proof}

Now, applying Lemma~\ref{lemma:competitive}, we complete our competitive  analysis by proving Theorem~\ref{thm:competitive}.

\begin{proof}[Proof of Theorem~\ref{thm:competitive}]
Continuing from Theorem~\ref{thm:blackbox_stability}, it shows that if the Lipschitz constant $C_{\ell}$ satisfies $C_{\ell}<\frac{1-\rho}{C_F(1+\|K\|)}$ where $\left\| F^t\right\|\leq C_F\rho^t$ for any $t\geq 0$ and the consistency error $\varepsilon$ satisfies $\varepsilon<\min\left\{ \frac{\sigma}{2\|H\|},\frac{1/C_{F}-C^{\mathsf{sys}}_{a}C_{\ell}}{C_{\ell}+\|B\|}\right\}$ where $C^{\mathsf{sys}}_{a}$ is defined in~\eqref{eq:policy_constant}, then  
$$
\|x_t\|\leq   \frac{\gamma^t-\mu^t}{1-\mu\gamma^{-1}}\left({C_F}+\mu\gamma^{-1}\right)\|x_0\|,
$$
for all $t\geq 0$ with some $\mu<1$. Therefore, Lemma~\ref{lemma:competitive} implies Theorem~\ref{thm:competitive}.
\end{proof}

\end{document}